\documentclass[10pt,twocolumn,letterpaper]{article}

\usepackage{iccv}
\usepackage{times}
\usepackage{epsfig}
\usepackage{graphicx}
\usepackage{amsmath}
\usepackage{amssymb}
\usepackage{epstopdf}
\usepackage{multirow}
\usepackage{subfig}
\usepackage{algorithm}
\usepackage{algorithmic}
\usepackage{authblk}
\usepackage{MnSymbol,wasysym}
\usepackage{booktabs} % Nice tables
\usepackage[table,dvipsnames]{xcolor}
\usepackage{pifont}
\usepackage{comment}
\usepackage{enumitem}

\usepackage{graphicx}
\usepackage{sidecap}

\graphicspath{ {/scratch/myData/zurichMap/results/} }
\usepackage[export]{adjustbox}

\usepackage{mathtools}

\newtheorem{theorem}{Theorem}[section]

\newtheorem{proposition}[theorem]{Proposition}
\newtheorem{corollary}[theorem]{Corollary}
\newtheorem{definition}[theorem]{Definition}

\newtheorem{problem}[theorem]{Problem}

\newenvironment{proof}[1][Proof]{\begin{trivlist}
\item[\hskip \labelsep {\bfseries #1}]}{\end{trivlist}}
\newcommand{\norm}[1]{\left\lVert#1\right\rVert}
\newcommand{\qed}{\nobreak \ifvmode \relax \else
      \ifdim\lastskip<1.5em \hskip-\lastskip
      \hskip1.5em plus0em minus0.5em \fi \nobreak
      \vrule height0.75em width0.5em depth0.25em\fi}

\usepackage[pagebackref=true,breaklinks=true,letterpaper=true,colorlinks,bookmarks=false]{hyperref}

\newcolumntype{C}[1]{>{\centering\let\newline\\\arraybackslash\hspace{0pt}}m{#1}}
\newcolumntype{L}[1]{>{\raggedright\let\newline\\\arraybackslash\hspace{0pt}}m{#1}}

\iccvfinalcopy % *** Uncomment this line for the final submission

\def\iccvPaperID{5062} % *** Enter the ICCV Paper ID here

\ificcvfinal\fi

\begin{document}

%%%%%%%%% TITLE
%\title{ Robust Matching for Two Views with Unknown Camera Model and Motion}
%\title{Camera Model and Motion Discovery from Feature Correspondences}
\title{Convex Relaxations for Consensus and Non-Minimal Problems in 3D  Vision}
\author{Thomas Probst} 
\author{Danda Pani Paudel}
\author{Ajad Chhatkuli}
\author{Luc Van Gool}
\affil{Computer Vision Laboratory, ETH Zurich, Switzerland}

%% Five different problems will be solved in this manner!
\maketitle
\thispagestyle{empty}

%%%%%%%%% ABSTRACT
\begin{abstract}
In this paper, we formulate a generic non-minimal solver using the existing tools of Polynomials Optimization Problems (POP) from computational algebraic geometry. The proposed method 
exploits the well known  Shor's or Lasserre's relaxations, whose theoretical aspects are also discussed. Notably, we further exploit the POP formulation of non-minimal solver also for the generic consensus maximization problems in 3D vision.
Our framework is simple and straightforward to implement, which is also supported by three diverse applications in 3D vision, namely rigid body transformation estimation, Non-Rigid Structure-from-Motion (NRSfM), and camera autocalibration. In all three cases, both non-minimal and consensus maximization are tested, which are also compared against the  state-of-the-art methods. Our results are competitive to the compared methods, and are also coherent with our theoretical analysis. 

%For the first time in this paper, we formulate the consensus for camera autocalibration as a Mixed Integer Program (MIP), and provide a tractable consensus maximization method for isometric NRSfM. However, these are not our main contributions. 

The main contribution of this paper is the claim that a good approximate solution for many polynomial problems involved in 3D vision can be obtained using the existing theory of numerical computational algebra. This claim leads us to reason about why many relaxed methods in 3D vision behave so well? And also allows us to offer a generic relaxed solver in a rather  straightforward way. We further show that the convex relaxation of these polynomials can easily be used for maximizing consensus in a deterministic manner. We support our claim using several experiments for aforementioned three diverse problems in 3D vision.  
\end{abstract}

%-- Will be written later --
\section{Introduction}
Robust model estimation in 3D vision usually involves two sub-problems: inlier detection and model computation. Faithful methods,  already existing in many cases, have offered an unparalleled success of 3D vision applications from  3D scene reconstruction to 6DoF camera localization. 
In general, the task of inlier detection is performed by maximizing the consensus among measurements. Measurements that agree with some parameter of the known model are treated as inliers. Customarily, model parameters are estimated from randomly selected minimal number of measurements -- using the so-called minimal solver methods. These parameters then seek for the maximum consensus, within the framework of Random Sampling and Consensus (RANSAC). Parameters that maximize the consensus are further refined with respect to all the inlier measurements -- using  the so-called non-minimal solver methods.

Global methods for consensus maximization have
gathered significant attention~\cite{li2009consensus, bazin2014globally, chin2015efficient, cai2018deterministic, paudel2018sampling, yang2014optimal, pani2015robust},  mainly because RANSAC is non-deterministic and provides no guarantee for optimality. Although the consensus maximization problem has proven to be NP-hard~\cite{chin2018robust}, global methods with convex model representations have shown promising results both in terms of speed and optimality~\cite{speciale2017consensus, chin2015efficient}. Therefore, existing methods offer satisfactory solutions only for convex models. In the case of non-convex models, some methods introduce linearized auxiliary model parameters with additional convex constraints~\cite{speciale2018consensus, speciale2017consensus}. There do exist  methods for non-convex models which explore the exact parameters~\cite{yang2014optimal, paudel2018sampling, pani2015robust}. However, 
all methods designed for non-convex problems are specific to the problems at hand. These methods therefore do not generalize for other non-convex models. Consequently, a generic framework for deterministic consensus maximization of non-convex models, beyond simple linearization, is highly desirable. 

Unlike consensus maximization, non-minimal solvers have gathered comparatively little attention. In practice, the need of optimal non-minimal solvers is usually overlooked.
Instead, a local refinement on some geometric cost derived for all inliers is performed as a surrogate method, starting from the model parameters obtained by maximizing consensus. This however entails the risk of falling into local minima traps of the global cost function resulting in inadequate solutions, which has been demonstrated for various problems~\cite{briales2017convex, olsson2009branch, kneip2014upnp, briales2018certifiably, RosenCBL19, ErikssonOKC18, LajoieHBC19}. More specifically, non-minimal solvers have been devised separately for 3D rotations~\cite{ErikssonOKC18}, 3D rigid body transformation in~\cite{briales2017convex, olsson2009branch, RosenCBL19, LajoieHBC19}, perspective-n-point in~\cite{kneip2014upnp}, and two-view relative pose in~\cite{briales2018certifiably}. On the one hand, these solvers are very specific to their own addressed problems, and only~\cite{olsson2009branch} provides the theoretical optimality guarantee owing to the used Branch-and-Bound (BnB) search paradigm.  On the other hand, methods in~\cite{kneip2014upnp,briales2017convex,briales2018certifiably} are fast and also achieve a-posteriori certificate,  with an open question for theoretical proof, of the global optimality. One notable work~\cite{kahl2007globally}, provides convex relaxations for multiple computer vision problems while primarily focusing on the task of geometric error minimization, where the addressed problems  are often linear with respect to the algebraic error minimization.

In fact, minimization of cost functions defined by polynomials, the global cost function in the non-minimal case,  is a long standing problem in algebraic geometry. Most notably, one can use the Sum-of-Squares (SoS) methods to minimize such polynomials using  Semi-Definite Programming (SDP). The theoretical optimality in this process can be guaranteed, when a hierarchy of SoS polynomials as constraints~\cite{parrilo2000structured} is enforced. In this context, many methods in 3D vision make use of the SoS hierarchy; among which are non-rigid 3D reconstruction~\cite{parashar2016isometric,bartoli2017generalizing}, camera  auto-calibration~\cite{Chandraker:2007,pani2015robust}, 3D-3D registration~\cite{pani2015robust}, and 3D modeling~\cite{ahmadi2016geometry}, to name a few. The SoS hierarchy is also known as the Lasserre's hierarchy of relaxations in SDP~\cite{lasserre2001global}. The first term of Lasserre's relaxations is indeed the so-called Shor's relaxation, a well known result from the 80's~\cite{shor1987quadratic}, which is also known to be tight for quadratic polynomial optimization~\cite{boyd1997semidefinite}. Therefore, it is not very surprising that the non-minimal solvers of~\cite{briales2017convex,briales2018certifiably}, which use Shor's relaxation with additional problem specific constraints, are fast and accurate. Moreover, in the context of non-minimal solvers, where one is offered sufficiently many polynomials, the tightness obtained using only the Shor's relaxation may not be an issue. Motivated by these observations, we are interested to study the behaviour of SoS hierarchy for both consensus maximization and non-minimal problems.       

In this paper, we formulate a generic non-minimal solver using the existing tools of Polynomials Optimization Problems (POP) from computational algebraic geometry. We also provide theoretical insights for using Shor's or Lasserre's relaxation, and support them by synthetic and real data experiments. More precisely, we suggest that
Shor's relaxation can very often offer satisfactory solutions when the models can be expressed as  quadratic polynomials. Furthermore, we show that higher order of Lasserre's relaxation is required for non-minimal solvers for higher degree polynomials in the case of challenging non-rigid 3D reconstruction. More interestingly, we successfully apply the same formulation used for non-minimal solvers also in the consensus maximization context.
%Since the suggested relaxation is convex, they can be easily integrated within the Branch-and-Bound framework for consensus maximization.
It then becomes obvious to make these relaxations tighter by adding more terms of Lasserre's relaxation,  as one desires.  However, we argue that one may not always need to use higher order relaxation when seeking for consensus among polynomials.

Our framework is simple and straightforward to implement, which is supported by three diverse applications in 3D vision, namely rigid body transformation estimation, non-rigid reconstruction, and camera autocalibration. In all three cases, both non-minimal and consensus maximization are tested. We further show that the suggested relaxations work satisfactorily even for the  minimal problem setup. All of our experiments are compared against the leading state-of-the-art methods. Our results are competitive to the compared methods, both in speed and accuracy,  which are also coherent with our theoretical analysis. 

%The paper is organized as follows. In Sec.~\ref{sec:background}, we introduce  theoretical background using the standard POP formulation. In Sec.~\ref{sec:nonMinProblems} we provide the non-minimal problem formulation 
%and suggest solutions for models represented by quadratic and higher degree polynomials, followed by the consensus maximization case in Sec.~\ref{sec:conMax}.  Sec.~\ref{sec:visionProblems} formulates three different problems in 3D vision, with their results in Sec.~\ref{sec:results}, after a brief discussion in Sec.~\ref{sec:Discussion}. Sec. 7 concludes our work with future perspectives.    

\section{Background and Notations}\label{sec:background}
We denote matrices with upper case letters and their elements by double-indexed lower case letters: $\mathsf{A}=(a_{ij})$. Similarly, vectors are indexed: $\mathsf{a}=(a_{i})$. We write
$\mathsf{A}\succ0$ (resp. $\mathsf{A}\succeq0$) to denote that the symmetric matrix $\mathsf{A}$ is positive definite (resp. positive semi-definite).  
Let ${\rm I\!R}[\mathsf{x}]$ be the ring of polynomials parametrized by variables $\mathsf{x} \in {\rm I\!R}^n$, and $\mathsf{z}_d(\mathsf{x})$ be 
a vector of monomials, on the entries of $\mathsf{x}$, ascending in degree up to  $d$. Any polynomial $p(\mathsf{x})\in{\rm I\!R}[\mathsf{x}]$ can be represented using
$\mathsf{z}_d(\mathsf{x})$ and the Gram matrix. 

\begin{definition}[Gram matrix~\cite{Powers1998}]\label{def:gram-matrix}
For a degree $\leq 2d$  polynomial $p(\mathsf{x})\in{\rm I\!R}[\mathsf{x}]$, 
the symmetric matrix $\mathsf{G}$ such that $p(\mathsf{x})=\mathsf{z}_d(\mathsf{x})^\intercal \mathsf{G}\mathsf{z}_d(\mathsf{x}) $ is a Gram matrix of $p(\mathsf{x})$.
\end{definition}
One is often interested on solving the following general non-convex polynomial optimization problem.

\begin{problem} [Polynomial optimization problem~\cite{lasserre2000convergent}]
\label{pb:POP}
For a set of general non-convex polynomials $p_i(\mathsf{x}), i=0,\ldots,m$, the Polynomial Optimization Problem (POP) is given by,
\begin{equation}
\underset{\mathsf{x}}{\text{min}\,\,}
\Big\{p_0(\mathsf{x}) |\,\, p_i(\mathsf{x})\geq 0, i=1,\ldots,m \Big\}.
\label{eq:gqpProb}
\end{equation}
\end{problem}
In the following, we first focus on the cases when $p(\mathsf{x})$ is quadratic before diving into higher degree polynomials. Notice that for quadratic polynomials,  $\mathsf{z}_1(\mathsf{x})$ is the homogeneous representation of $\mathsf{x}$. In such case, a relaxed convex solution can be obtained using the Shor's method.

\begin{definition}[Shor's relaxation~\cite{shor1987}]\label{def:shorRelax} For general non-convex quadratic polynomials $p_i(\mathsf{x}), i=0,\ldots,m$, the
POP of~\eqref{eq:gqpProb}  is equivalent to the following problem.
\begin{equation}
\underset{\mathsf{Y}\succeq0}{\text{min}\,\,} \Big\{\text{tr}(\mathsf{G}_0\mathsf{Y})|\,\, \text{rk}(\mathsf{Y})=1, \text{tr}(\mathsf{G}_i\mathsf{Y})\geq 0,  i=1,\ldots, m \Big\},
\label{eq:shorrlx}
\end{equation}
for rank and trace operators $\text{rk}(.)$ and $\text{tr}(.)$, respectively.
A convex relaxation of~\eqref{eq:shorrlx} can be obtained by dropping the rank-1 constraint on  $\mathsf{Y}\colon = \mathsf{z}_1(\mathsf{x})\mathsf{z}_1(\mathsf{x})^\intercal$.
%by  $\mathsf{Y}\succeq0$. 
\end{definition}
It is well known that the Shor's relaxation is tight for the setup of~\eqref{eq:shorrlx}~\cite{boyd1997semidefinite,lasserre2002semidefinite}. Therefore, the solution of~\eqref{eq:gqpProb} is very well approximated by such relaxation. Although, we develop most of the theory using the quadratic polynomials with Shor's relaxation, we also address the case of higher degree polynomials with tighter relaxations.  In the latter, we make use of a relaxation based on Lasserre's hierarchy~\cite{lasserre2001global}, similar to  the Shor's relaxation for quadratic case.

Now, we briefly present the theory behind  Lasserre's relaxation.
In particular, we discuss the case when the highest order of relaxation is $2n$.
For a vector of relaxed variables $\mathsf{w}$, the truncated moment matrix of  order $2n$ is $\mathsf{M}^{2n}(\mathsf{w})= \mathsf{W}=\mathsf{W}^\intercal$ such that $w_{\alpha\beta}=w_{\alpha+\beta}$ where $\alpha,\beta\in{\rm I\!N}^n$. By construction, any lower order moment matrix is a submatrix of $\mathsf{W}$ where $\mathsf{M}^n(\mathsf{w})= \mathsf{Y}\colon = \mathsf{z}_d(\mathsf{x})\mathsf{z}_d(\mathsf{x})^\intercal$ and $\mathsf{M}^0(\mathsf{w})=1$. Lasserre's method derives a hierarchy of constraints using the so-called Localizing matrices.

\begin{definition}[Localizing matrix~\cite{lasserre2001global}]\label{def:lasserreRelax}
The localizing matrix for polynomial $p(\mathsf{x})$ and relaxation order $s\leq n$ is a matrix $\mathsf{M}^s(p(\mathsf{x}) \mathsf{w})$ given by, 
\begin{equation}
     \mathsf{M}^s(p(\mathsf{x}) \mathsf{w}) = \mathcal{L}(\text{tr}(\mathsf{GY})\mathsf{M}^s(\mathsf{w})),
\end{equation}
where  the Riesz function $\mathcal{L}(.)$ maps the bilinear terms on  $\mathsf{Y}$ and $\mathsf{W}$ to the corresponding terms of $\mathsf{W}$.  Note that $\mathsf{M}^s(\mathsf{w})$ is a submatrix of $\mathsf{Y}$ for $s\leq n$. Therefore,  $\mathsf{M}^s(p(\mathsf{x}) \mathsf{w})$ can be expressed linearly on $\mathsf{W}$ in this case. 
\end{definition}
\begin{definition}[ Lasserre's relaxation~\cite{lasserre2001global}]\label{def:lasserreRelax}
An efficient relaxed solution of Problem~\ref{pb:POP} of general non-convex polynomials can be obtained by solving a Semi-Definite Program (SDP) of a hierarchy  of relaxations. For the relaxation order of $2n$ and $i = 1,\ldots, m$, it is given by,  
\begin{equation}
\underset{\mathsf{W}\succeq 0}{\text{\,\,min\,\,}}\Big\{
\text{tr}(\mathsf{G}_0\mathsf{M}^n(\mathsf{w}))|\,\,\text{tr}(\mathsf{M}^s(p_i(\mathsf{x})\mathsf{w}))\succeq0\Big\}. 
\label{eq:relaxLasser}
\end{equation}
\end{definition}
The Lasserre's relaxation is known to be tighter than Shor's relaxation, which also provides the certificate of finite convergence when $s\geq n$ ~\cite{lasserre2001global}. In fact, the Lasserre's relaxation includes the Shor's relaxation as a special case when $s=0$ because $\mathsf{M}^0(p_i(\mathsf{x})\mathsf{w})=\text{tr}(\mathsf{G}_i\mathsf{Y})$ and 
$\mathsf{M}^n(\mathsf{w})= \mathsf{Y}$. Needless to say that the Lasserre's relaxation is also applicable to  quadratic polynomials at the cost of a higher computation. However, for the non-minimal problems of quadratic polynomials, our experiments show that the Lasserre's relaxation is not really necessary.

\section{Non-minimal Problem of Polynomials}\label{sec:nonMinProblems}
In this section, we formulate the non-minimal problem of polynomials.
We first propose a solution to this problem for general quadratic polynomials using the  Shor's relaxation method. Later, the proposed solution will be generalized for the case of 
higher degree polynomials using the Lasserre's method of relaxation.

\begin{problem} \label{pb: nonMinProb}
For $\mathsf{x}\in{\rm I\!R}^n$ and $m\geq n$, the  non-minimal problem of  
 polynomials $p_i(\mathsf{x}) \in {\rm I\!R}[\mathsf{x}], i=1,\ldots,m$, is,
\begin{equation}
\underset{\mathsf{x},\epsilon_i}{\text{min\,\,}}\Big\{
\sum_i{\epsilon_i}|\,\, \epsilon_i\geq \lvert p_i(\mathsf{x})\rvert, i = 1,\ldots, m\Big\}. 
\label{eq:nonMinProb}
\end{equation}
\end{problem}
The problem of~\ref{pb: nonMinProb} is in fact the  $\mathcal{L}_1$ minimization problem over polynomials. One is often interested to minimize such objective because in many 3D vision problems  $\epsilon$ is a geometric measure (such as point-to-line distance for two-view epipolar constraint using Essential matrix).  Depending upon application, one may be interested to minimize $\mathcal{L}_2$ (or $\mathcal{L}_\infty$  for that matter). This nonetheless, is not really a problem. For the sake of simplicity, we first  present the theory using the $\mathcal{L}_1$ formulation. Its extension to $\mathcal{L}_p-\text{norm}$ is discussed in  Section~\ref{sec:Discussion}. 
Non-minimal problems are the outcome of over-detemined systems. 
The solution we are seeking is the one that agrees with all polynomials with minimum cumulative error, unlike the exact solution of the minimal case. Note when $n=m$, the non-minimal problem becomes a minimal problem.

\subsection{Quadratic Polynomials}
\begin{proposition} \label{pr: nonMinProbQuad}
For a set of non-convex quadratic polynomials $\{p_i(\mathsf{x})\}_{i=1}^m$, the Shor's relaxation  provides a convex relaxation of the non-minimal Problem~\ref{pb: nonMinProb} as,
\begin{equation}
\underset{\mathsf{Y}\succeq 0,\epsilon_i}{\text{\,\,min\,\,\,\,}}\Big\{
\sum_i{\epsilon_i}|\,\, \epsilon_i\geq \lvert tr(\mathsf{G}_i\mathsf{Y})\rvert, i = 1,\ldots, m\Big\}. 
\label{eq:nonMinProbQuad}
\end{equation}
\end{proposition}
\begin{proof}
Here, we provide the intuition behind the proof.
Note that the problem of~\eqref{eq:nonMinProbQuad} is equivalent to~\eqref{eq:nonMinProb} for quadratic polynomials $\{p_i(\mathsf{x})\}_i^m$ when $\text{rk}(\mathsf{Y})=1$. Therefore~\eqref{eq:nonMinProbQuad} can directly be obtained by dropping the rank constraint, similar as in the Shor's relaxation~\ref{def:shorRelax}. In fact, a rigorous proof can be obtained using the POP formulation of Section~\ref{sec:background}. Please, refer to the supplementary for the alternative proof. \hfill \qed      
\end{proof}
The relaxed convex problem~\eqref{eq:nonMinProbQuad} is an SDP.
This can  be solved efficiently  using the interior point method~\cite{boyd}. Ideally, $\mathsf{Y}$ is expected to be of rank-1. However, this is not usually the case. Therefore, we recover the primal solution $\mathsf{x}\in{\rm I\!R}^n$ after enforcing the rank-1 constraint using Singular Value Decomposition (SVD). In fact, the principal singular-vector $\mathsf{Y}$ is the homogeneous representation of  $\mathsf{x}\in{\rm I\!R}^{n+1}$. Recall that the dual relaxed variable $\mathsf{Y}$ encodes the primal solution in the form $\mathsf{Y}\colon=z_1(\mathsf{x})z_1(\mathsf{x})^\intercal$. 

\subsection{General Polynomials}
\begin{proposition} \label{pr: nonMinProbGen}
For a set of non-convex general polynomials $\{p_i(\mathsf{x})\}_{i=1}^m$ of degree $\leq d$, the Lasserre's relaxation with order $s\leq n\in{\rm I\!N}$  provides a convex relaxation  of the non-minimal Problem~\ref{pb: nonMinProb} as,
\begin{equation}
\underset{\mathsf{W}\succeq 0,\epsilon_i}{\text{\,\,min\,\,\,\,}}\Big\{
\sum_i{\epsilon_i}|\,\, \epsilon_i\geq \lvert \text{tr}(\mathsf{M}^s(p_i(\mathsf{x})\mathsf{w}))\rvert, i = 1,\ldots, m\Big\}. 
\label{eq:nonMinProbGen}
\end{equation}
\end{proposition}
\begin{proof}
The proof is similar to the Proposition~\ref{pr: nonMinProbQuad}. Please, refer the supplementary material for the details.\hfill \qed
\end{proof}
The optimal primal solution $\mathsf{x}\in{\rm I\!R}^{n}$ can be recovered from $\mathsf{W}$ using SVD, similar to that of~\eqref{eq:nonMinProbQuad} discussed above.

\section{Consensus Maximization Problem}\label{sec:conMax}
The non-minimal method presented in the previous section assumes that the polynomials may be corrupted by the noise. Therefore, a solution that minimizes the cumulative error of all polynomials is searched. 
In the presence of noise and outliers, we wish to solve the following Problem.

\begin{problem} \label{pb:conMax}
Given a set  ${\mathcal{S} = \{p_i(\mathsf{x})\}_{i=1}^{m}}$ and a threshold $\epsilon$, 
\label{pb:prob1}
\begin{equation}
\underset{\mathsf{x},\mathcal{\zeta} \subseteq \mathcal{S}}{\,\,\,\,\text{max}\,\,\,\,} \Big\{|\mathcal{\zeta}|, \,\, \epsilon\geq\lvert p_i(\mathsf{x})\rvert, \,\, \forall p_i(\mathsf{x})\in\mathcal{\zeta}\Big\}.
\label{eq:samConMax}
\end{equation}
\end{problem}
The consensus maximization problem seeks for a feasible solution that results the largest inlier set -- a set of polynomials with residual smaller than $\epsilon$.
This problem, however, is difficult to solve and known to be NP-hard~\cite{chin2018robust}, even when $p(\mathsf{x})$ is a linear function on $\mathsf{x}$. In this work, we approach this problem using the Branch-and-Bound search paradigm. 

\subsection{Branch-and-Bound Method}
Our Branch-and-Bound(BnB) search is performed by branching on the space of binary assignment variables, one for each member in $\mathcal{S}$. During the BnB process, we seek for a feasible $\mathsf{x}$ using the mixed-integer programming method.
\begin{definition} [Mixed-Integer Programming, MIP]
For a set of binary variables $\mathcal{Z}\in\{0,1\}^m$ representing the inlier/outlier assignments, a 
given set of polynomials ${\mathcal{S} = \{p_i(\mathsf{x})\}_{i=1}^{m}}$ and a threshold $\epsilon$, the consensus maximization Problem~\ref{pb:conMax} is equivalent to,
\label{pb:prob1}
\begin{equation}
\underset{\mathsf{x},z_i\in\mathcal{Z}}{\,\,\,\,\text{min}\,\,\,\,} \Big\{\sum_{i=1}^m{z_i}| \,\, z_i\mathbf{M}+ \epsilon\geq\lvert p_i(\mathsf{x})\rvert, \,\,\mathcal{Z}\in\{0,1\}^m\Big\},
\label{eq:samConMax}
\end{equation}
where $\mathbf{M}$ is a sufficiently large positive scalar constant, commonly used in optimization to ignore invalid constraints~\cite{chinneck2007feasibility}. If the binary variable $z_i=0$, the polynomial $p_i(\mathsf{x})$ is an inlier. Otherwise, it is an outlier.
\end{definition}

\subsection{Consensus using MI-SDP}
We formulate the consensus maximization problem by using the polynomial relaxations within the framework of Mixed-Integer Semi-Definite Programming (MI-SDP), for two different cases. First, we use Shor's relaxation of~\eqref{eq:shorrlx} with MI-SDP for solving the consensus maximization problem of quadratic polynomials. Later, the Lasserre's relaxation of~\eqref{eq:relaxLasser} is used within the same framework, for the consensus maximization  of more general polynomials.
\begin{corollary}[Quadratic case]\label{col:qadriaticCase}
Given a set of non-convex quadratic polynomials ${\mathcal{S} = \{p_i(\mathsf{x})\}_{i=1}^{m}}$ and a threshold $\epsilon$, the consensus maximization Problem~\ref{pb:conMax} can be solved using the following Mixed integer semi-definite program,
\label{pb:prob1}
\begin{equation}
\underset{\mathsf{Y}\succeq0,z_i\in\mathcal{Z}}{\,\,\,\,\,\,\text{min}\,\,\,\,} \Big\{\sum_{i=1}^m{z_i}| \,\, z_i\mathbf{M}+ \epsilon\geq\lvert \text{tr}(\mathsf{G}_i\mathsf{Y})\rvert, \,\,\mathcal{Z}\in\{0,1\}^m\Big\}.
\label{eq:samConMaxQad}
\end{equation}
\end{corollary}

The MI-SDP of~\eqref{eq:samConMaxQad} can be solved efficiently using off-the-self optimization toolboxes. The  MI-SDP solution provides us the optimal set of inlier polynomials, together with a feasible $\mathsf{x}$. These inlier polynomials are then used to solve the non-minimal problem of~\eqref{eq:nonMinProbQuad}, to obtain an optimal solution $\mathsf{x}$. A similar MI-SDP formulation for the consensus maximization among general polynomials is given below. 

\begin{corollary}[General case]\label{col:generalCase}
Given a set of non-convex general polynomials ${\mathcal{S} = \{p_i(\mathsf{x})\}_{i=1}^{m}}$ and a threshold $\epsilon$, the consensus maximization Problem~\ref{pb:conMax}  for assignments $\mathcal{Z}\in\{0,1\}^m$ can be solved using the following MI-SDP,
\label{pb:prob1}
\begin{equation}
\underset{\mathsf{W}\succeq0,z_i\in\mathcal{Z}}{\,\,\,\,\,\,\text{min}\,\,\,\,} \Big\{\sum_{i=1}^m{z_i}| \,\, z_i\mathbf{M}+ \epsilon\geq\lvert \text{tr}(\mathsf{M}^s(p_i(\mathsf{x})\mathsf{w}))\rvert\Big\}.
\label{eq:samConMaxGen}
\end{equation}
\end{corollary}

\section{3D Vision Problems}\label{sec:visionProblems}
We present three examples of 3D vision problems for both consensus maximization and non-minimal problems.  We start by a generic formulation of consensus maximization for all three problems. The inlier set of consensus maximization are then used to solve a non-minimal problem for the optimal set of  parameters. During both of these  stages, we also introduce problem specific polynomial constraints. 

Let us consider a set of vectorized polynomials $\{P_i(\mathsf{x})\}_{i=1}^m$, possibly from many  outlier measurements, is given to us. In the first step, we are interested in solving the following consensus maximization problem,
    \begin{equation}
\underset{\mathsf{x},z_i\in\mathcal{Z}}{\,\,\text{min}\,\,\,\,} \Big\{\sum_{i=1}^m{z_i}| \,\, z_i\mathbf{M}+ \epsilon\geq\lvert P_i(\mathsf{x})\rvert, \mathsf{x}\in\mathcal{K},\mathcal{Z}\in\{0,1\}^m \Big\},
\label{eq:samConMaxRigid}
\end{equation}
where $\mathsf{x}\in \mathcal{K}$ represents problem specific constraints  and  $\mathcal{Z}$ measures inlier/outlier assignments,  and $''\geq''$ represents one-to-many elementwise inequality.  
We solve~\eqref{eq:samConMaxRigid} using our Corollary~\ref{col:qadriaticCase}/\ref{col:generalCase} for $\mathsf{x}$ and $\mathcal{Z}$. Once the inlier/outlier assignments, i.e. $\mathcal{Z}$, is obtained, we solve the following non-minimal problem using our Proposition~\ref{pr: nonMinProbQuad}/\ref{pr: nonMinProbGen}.
    \begin{equation}
\underset{\mathsf{x},\epsilon_i}{\,\,\,\,\text{min}\,\,\,\,} \Big\{\sum_{i=1}^m{\epsilon_i}| \,\, z_i\mathbf{M}+ \epsilon_i\geq\lvert P_i(\mathsf{x})\rvert, \mathsf{x}\in\mathcal{K} \Big\}.
\label{eq:nonMInRigid}
\end{equation}

\subsection{Rigid Body Transformation}
We consider correspondences between two point clouds that differ by a 3D rigid body transformation. Let $\{\mathsf{u},\mathsf{v}\}$ be euclidean coordinates of a pair of corresponding points such that $\mathsf{v}=\mathsf{Ru+t}$, for rotation matrix $\mathsf{R}\in SO(3)$ and translation $\mathsf{t}\in{\rm I\!R}^{3}$.  
For quaternions $\mathsf{q}\in{\rm I\!R}^{4}$, we represent the rotation matrix as $\mathsf{R(q)}$ with entries quadratic in $\mathsf{q}$. 
Given a set of correspondences $\{\mathsf{u}_i,\mathsf{v}_i\}_{i=1}^m$, we solve consensus maximization and non-minimal problems, respectively of \eqref{eq:samConMaxRigid} and \eqref{eq:nonMInRigid}, for variable $\mathsf{x}=(\mathsf{q}^\intercal,\mathsf{t}^\intercal)^\intercal$ and polynomials  $P_i(\mathsf{x})=\mathsf{v}_i-\mathsf{R}(\mathsf{q})\mathsf{u}_i-\mathsf{t}$, with $\mathcal{K}=\{\mathsf{x}\,|\norm{\mathsf{q}}^2=1\}$.

\subsection{Camera Autocalibration}
From a set of given Fundamental matrices $\{\mathsf{F}_i\}_{i=1}^m$, we wish to estimate the camera intrinsic matrix $\mathsf{K}$. Here, we assume that the camera intrinsic is constant across all the images involved during Fundamental matrix estimation. 
Let $\mathsf{\omega}$ be the Dual Image of Absolute Conic (DIAC) expressed in  $\mathsf{K}$ as $\mathsf{\omega} = \mathsf{K}\mathsf{K}^\intercal$.
The simplified Kruppa's equation~\cite{lourakis1999camera} allows us to express $\mathsf{\omega}$ in the form of polynomials using Fundamental matrices $\mathsf{F}_i$.
Let $\mathsf{F}_i = \mathsf{U}_i\mathsf{D}_i\mathsf{V}_i$ be the singular value decomposition, with $\mathsf{D} = \textnormal{diag}([r_i,s_i,0])$. For $\mathsf{U}_i = [\mathsf{u}_{i1}|\mathsf{u}_{i2}|\mathsf{u}_{i3}]$ and $\mathsf{V}{i} = [\mathsf{v}_{i1}|\mathsf{v}_{i2}|\mathsf{v}_{i3}]$, two independent polynomials of simplified Kruppa's equations are,
\begin{align}
\label{eq:simpliyKruppa}
P_{i1}(\mathsf{\omega})=(r_is_i\mathsf{v}_{i1}^\intercal\mathsf{\omega}\mathsf{v}_{i2})(\mathsf{u}_{i2}^\intercal\mathsf{\omega}\mathsf{u}_{i2})
+(r_i^2\mathsf{v}_{i1}^\intercal\mathsf{\omega}\mathsf{v}_{i1})(\mathsf{u}_{i1}^\intercal\mathsf{\omega}\mathsf{u}_{i2}),\nonumber
\\
P_{i2}(\mathsf{\omega})=
(r_is_i\mathsf{v}_{i1}^\intercal\mathsf{\omega}\mathsf{v}_{i2})(\mathsf{u}_{i1}^\intercal\mathsf{\omega}\mathsf{u}_{i1})
+(s_i^2\mathsf{v}_{i2}^\intercal\mathsf{\omega}\mathsf{v}_{i2})(\mathsf{u}_{i1}^\intercal\mathsf{\omega}\mathsf{u}_{i2}).
\end{align}
We parameterize $\mathsf{\omega}$ using $\mathsf{x}\in{\rm I\!R}^5$ because $\mathsf{\omega}$ is a $3 \times 3$ matrix with $\mathsf{\omega} = \mathsf{\omega}^\intercal$ and $\mathsf{\omega}_{(3,3)}= 1$. 
Given a set of Fundamental matrices $\{\mathsf{F}_i\}_{i=1}^m$, we solve consensus maximization and non-minimal problems of~\eqref{eq:samConMaxRigid} and~\eqref{eq:nonMInRigid}, for $\mathsf{x}\in{\rm I\!R}^5$,  ${P_i(\mathsf{x})=(P_{i1}(\mathsf{\omega}),P_{i2}(\mathsf{\omega}))^\intercal}$, and $\mathcal{K}=\{\mathsf{x}|\omega\succeq0\}$.
 The intrinsic $\mathsf{K}$ is recovered using the Cholesky decomposition on~$\mathsf{\omega}$. 

\subsection{Non-Rigid Structure-from-Motion (NRSfM)}\label{subSec:Nrsfm}
Method in~\cite{parashar2016isometric} proposes a framework of modeling NRSfM as a POP using the geometric prior of isometry. It models the rest shape as a Riemannian manifold and the deformed shapes as isometric mappings of the rest shape. The isometric deformation prior is then
expressed by relating the metric tensor, the Christoffel symbols parametrized by $k_1,\ k_2 \in\mathbb{R}$ and the inter-image registration warps in camera coordinates. We borrow the notations and definitions of the Christoffel symbol parameterization from the original problem formulation~\cite{parashar2016isometric}.
As in \cite{parashar2016isometric}, we summarize the resulting system of polynomials as,
\begin{equation}
\label{eq:iso-nrsfm}
    P_i(\mathsf{x}) = \mathcal{P}(\mathsf{x}) - \left\{\mathcal{P}_{1i}(\mathsf{x})\right\}_{i=2}^n
\end{equation}
\eqref{eq:iso-nrsfm} is in fact a system of $n-1$ independent quartic polynomials for $n$ images that relates the point-wise inter-image warp measurements ${\mathsf{q}}_{i=2}^n$ to the Christoffel symbols parameterized by $\mathsf{x} = [k_1\ k_2]^\top$. Solving
\eqref{eq:iso-nrsfm} amounts to solving the isometric NRSfM problem as $k_1,\ k_2$ can be used to obtain the Jacobian of each shape with respect to the reference image coordinates in the camera frame. One can then compute the surface normal at each corresponding point of the $n$ shapes. We solve consensus maximization and the non-minimal problem of \eqref{eq:iso-nrsfm} to solve the NRSfM problem.

\section{Discussion}\label{sec:Discussion}
The $\mathcal{L}_1$ constraint of~\eqref{eq:nonMinProb} (and the ones that follow) 
can be extended to $\mathcal{L}_\infty$ by replacing all $\epsilon_i$ variables by a single variable $\epsilon$, which is a common practice in convex optimization~\cite{boyd}. Similarly, $\mathcal{L}_2$ norm can also be minimized by imposing the conic constraints on stacked linear vector of polynomials $\mathsf{v}_i(\mathsf{x}) = [{\text{trace}(\mathsf{G}_{1i}\mathsf{Y})},\,\ldots, {\text{trace}(\mathsf{G}_{mi}\mathsf{Y})}]^\intercal$ of~\eqref{eq:nonMinProbQuad} (resp. of~\eqref{eq:nonMinProbGen}) for $m$ polynomials of the $i^{th}$ measurement such that $\norm{\mathsf{v}_i(\mathsf{x})}_2\leq \epsilon_i$.  
In this process, we introduce one variable for each measurement, therefore the time complexity of our method is expected to be linear on the number of measurements. Although these auxiliary variables may seem to introduce overhead, they are in fact helpful.
These variables allow us to express the non-minimal problems as an SOCP. The SOCP constraints enable us to easily extend our non-minimal formulation to that of consensus maximization, where variables $\epsilon_i$ naturally turn into inlier threshold. Alternatively for non-minimal problems, one could choose to minimize $\sum_i{\norm{\mathsf{v}_i(\mathsf{x})}^2}$ (or even $\sum_i{p_i(\mathsf{x})^2}$)  directly, using the standard SOS optimization methods. Such formulation not only compromises the flexibility, but also adds burden by increasing the degree of polynomials. 
\section{Experiments}\label{sec:results}
We conducted several experiments with synthetic and real datasets. The synthetic data was generated using the toolbox of \cite{briales2017convex}, in the very similar setups, whereas, the real datasets used are Fountain and Herz-Jesu~~\cite{Strecha2008}, and Flag~\cite{White2007}, Hulk and
Tshirt~\cite{parashar2016isometric}. In the case of real datasets, only the outliers were synthetically generated for quantitative evaluations. All our results for non-minimal solvers are generated after the consensus maximization, except for the minimal case. %Although, this may not be necessary when one is certain about the inlier correspondence. However, we do not assume that is the case and run the consensus maximization anyway. If there is no outlier in the measurements, the consensus maximization part of our framework quickly executes, without significant overload in terms of computation.
Our framework is implemented in MATLAB2015a and all the optimization problems are solved using MOSEK~\cite{mosek2012mosek}. All experiments were carried out on a 16GB RAM Pentium i7/3.40GHz. Some qualitative results obtained for different applications, discussed later, are shown in Figure~\ref{fig:qualitative_results}.

\begin{figure}
    \centering
 \includegraphics[height=2cm]{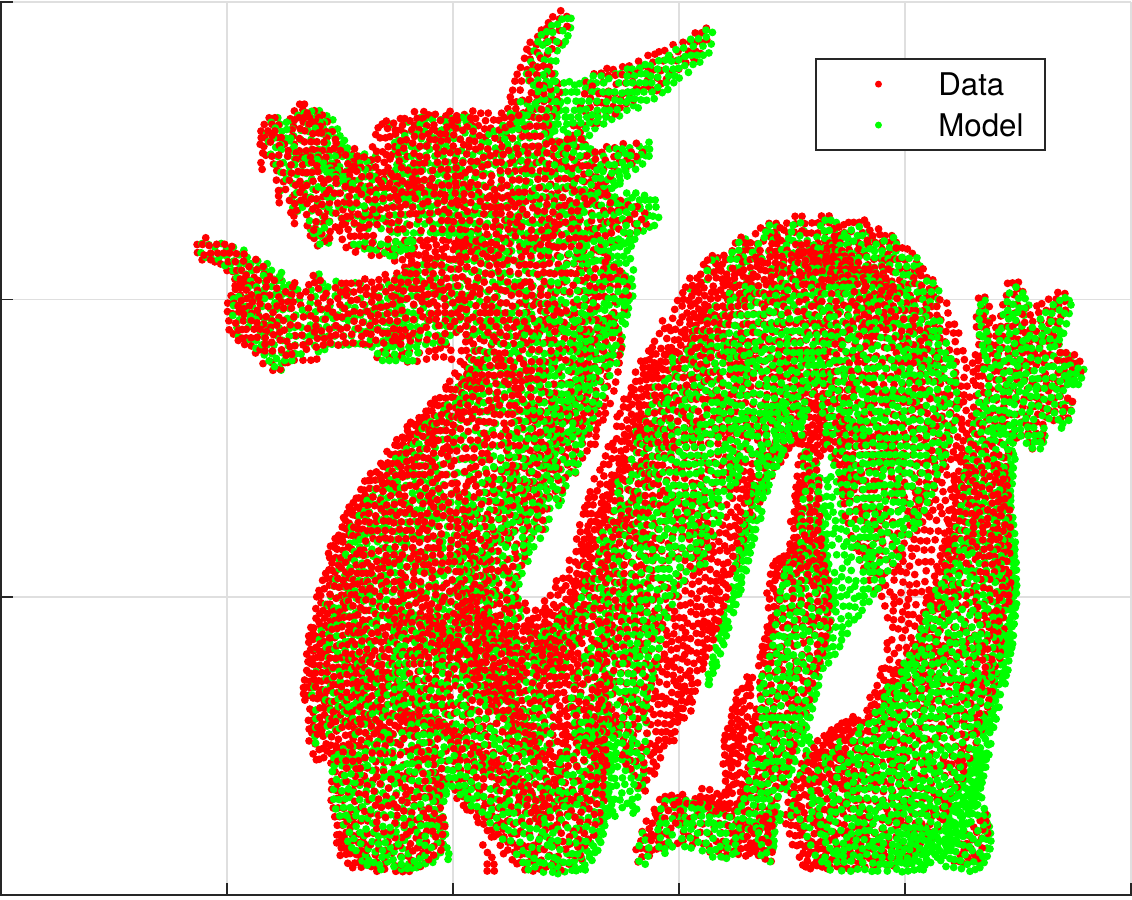}
 \includegraphics[height=2cm]{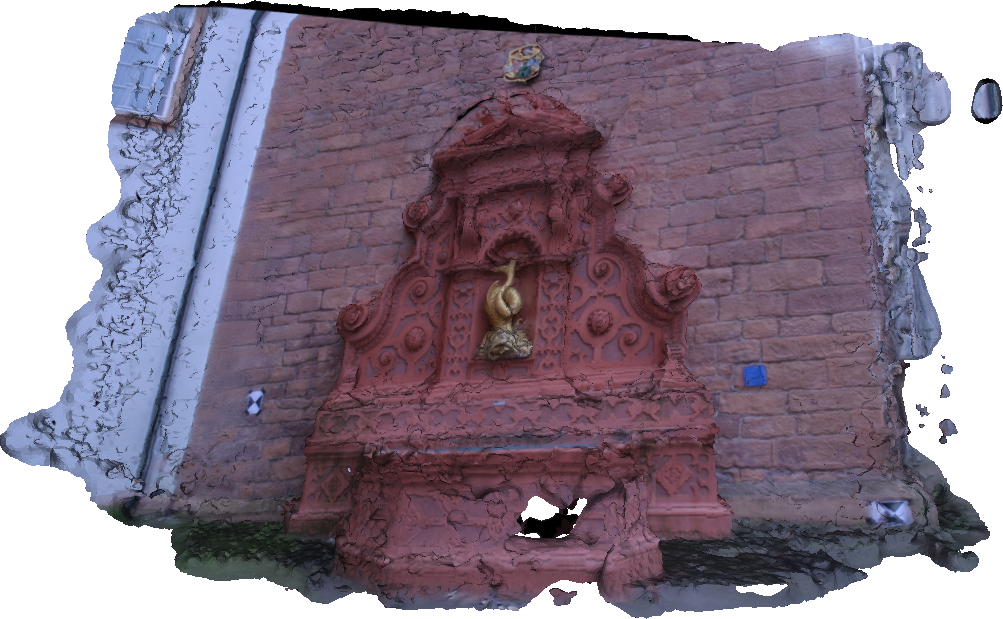}
 \includegraphics[height=2.3cm]{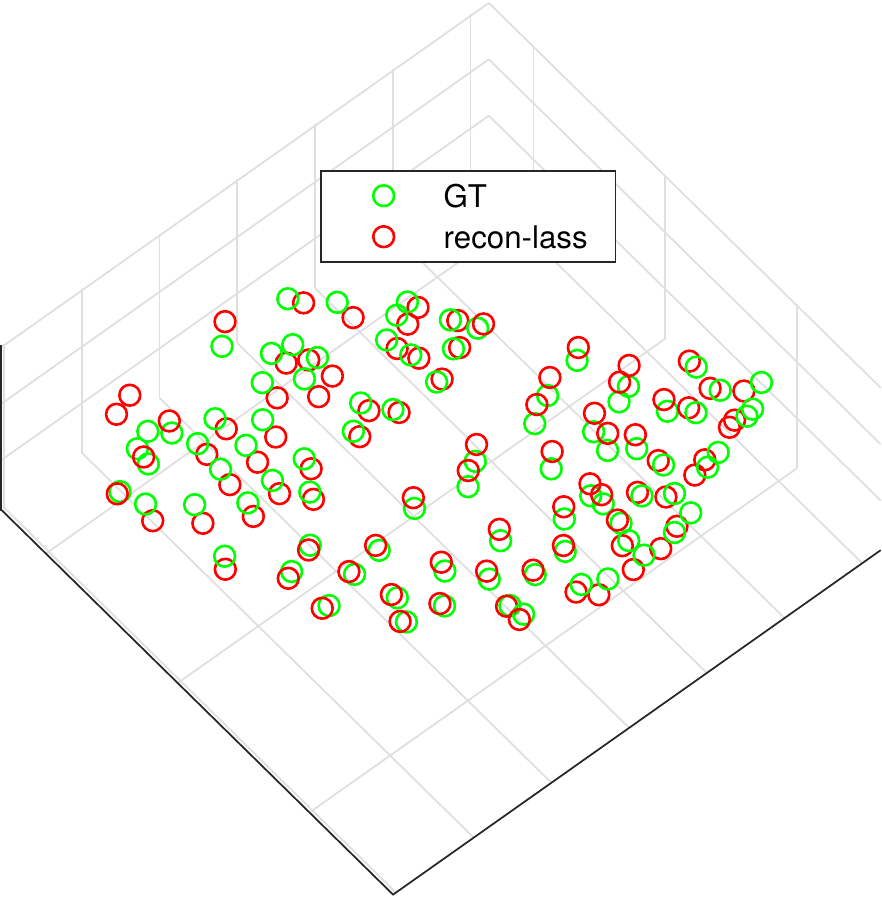}
    \caption{Non-minimal results obtained using our framework. From left to right: rigid registration (model: green; data: red) on the public dataset~\cite{chin2016guaranteed}; 3D reconstruction of the Fountain sequence obtained using the estimated camera intrinsics; non-rigid 3D reconstruction (ground truth: green; reconstruction: red) of tshirt  with Lasserre's relaxation.}
    \label{fig:qualitative_results}
\end{figure}

\subsection{Rigid Body Transformation}
We conducted a large amount of synthetic experiments to explore the behaviour of the proposed framework on rigid body transformation estimation, both for non-minimal and consensus maximization problems. This is mostly because there exists three different non-minimal solvers, namely \textbf{Briales'17}~\cite{barath2017minimal}, \textbf{Olsson'08}~\cite{olsson2008solving} and \textbf{Olsson-BnB}~\cite{olsson2009branch}, comparison against them is one of our interests. In the first experiment, we tested performance of different methods in the minimal setup with varying levels of noise. The obtained results are reported in Figure~\ref{fig:minimalRigid}. Note that the problem of rigid body transformation is a 3-point problem. One can observe for Figure~\ref{fig:minimalRigid} that our method performs very competitively against the globally optimal method \textbf{Olsson-BnB} in rotation and translation estimation, whereas, our method is very competitive in time, shown in Figure~\ref{fig:Time-non-minimalRigid} (left),  with \textbf{Briales'17}. Experiments with \textbf{Olsson'08} shows that this method does not support the minimal setup.

\begin{figure}
\includegraphics[width=4.2cm]{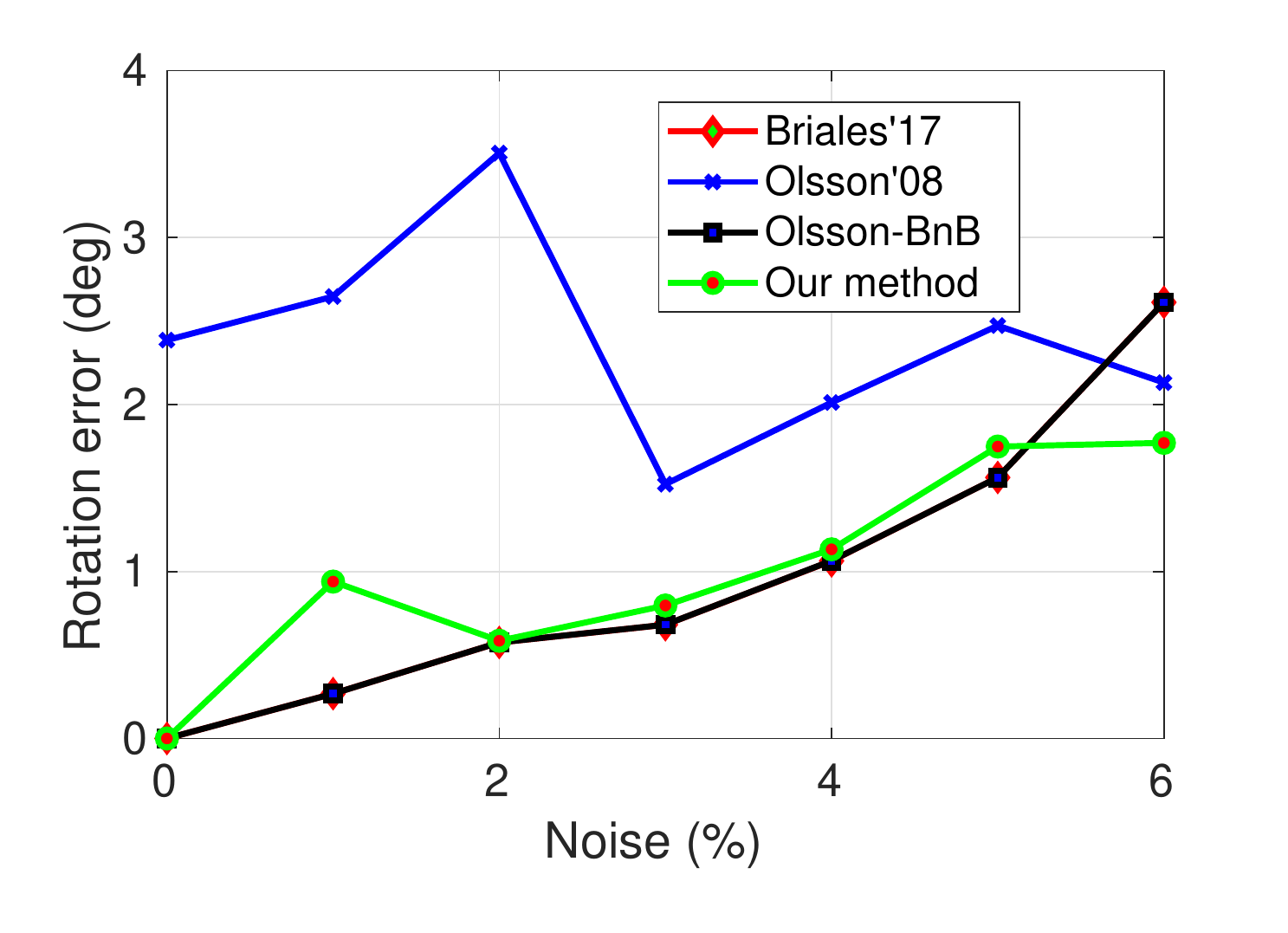}
\includegraphics[width=4.0cm]{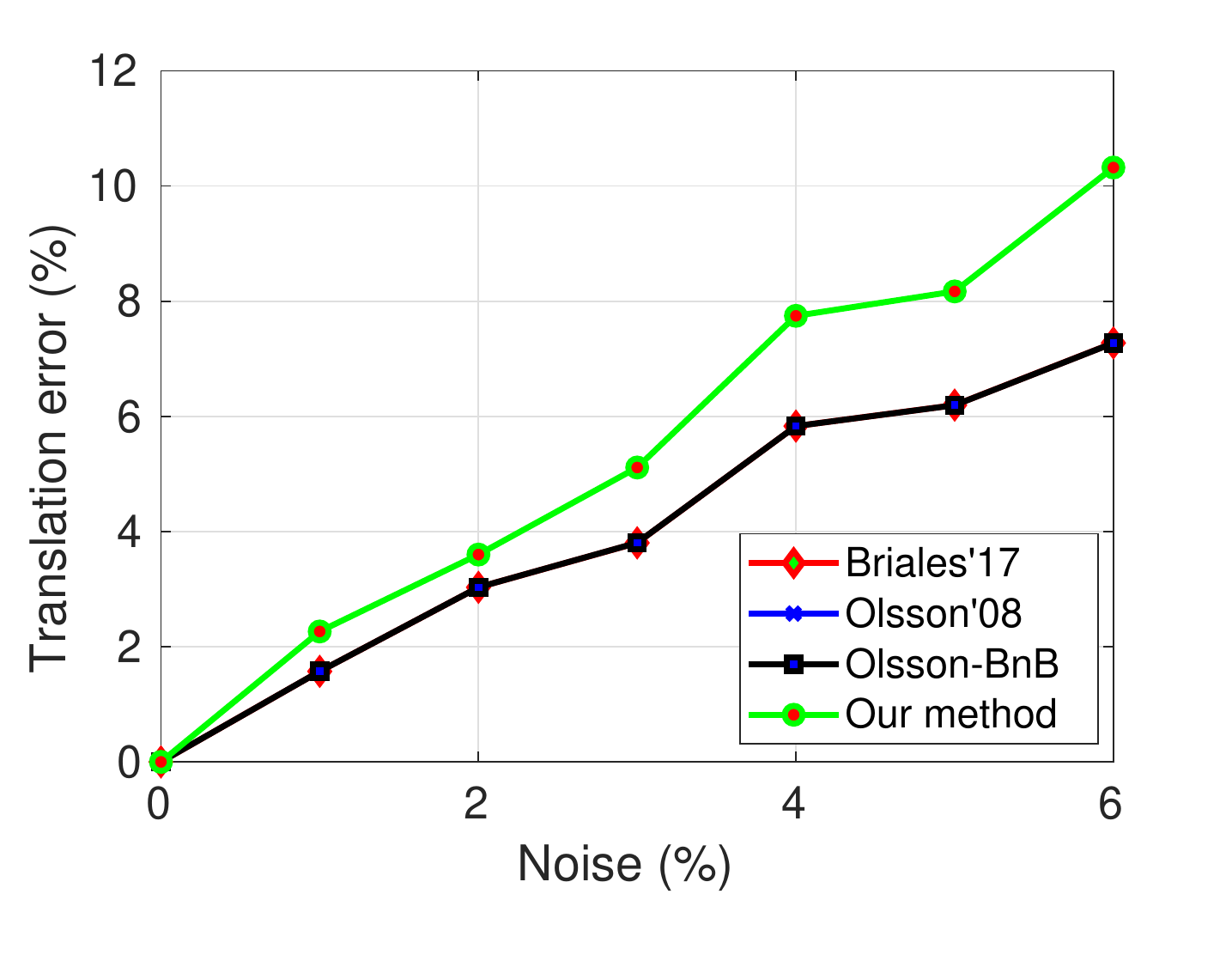}
%\vspace{-3mm
\caption{Results for minimial (3 points) rigid body transformation estimation. Left to right: noise vs. errors in rotation angles and translation for four different methods.\label{fig:minimalRigid}}
%\vspace{-1mm}
\end{figure}

Experiments for non-minimal setups were also conducted, with similar measures as that of minimal case, by varying the number of points for fixed noise level. These experimental results are reported in Figure~\ref{fig:non-minimalRigid}. Here we observe that our method still behaves similar to \textbf{Olsson-BnB} in accuracy. Note that the time, shown in Figure~\ref{fig:Time-non-minimalRigid} (right), for our method increases linearly with the number of points. The reason behind this has already been discussed in Section~\ref{sec:Discussion}, where we also suggest techniques 
to speed up by compromising the flexibility. 
%When we remove the variables introducing new conic constraints,
%Our method performs very similarly as the results reported in Figure~\ref{fig:non-minimalRigid}, while offering the speed reported in Figure~\ref{fig:minimalRigid}, if the collected $\mathcal{L}_2$ or $\mathcal{L}_\infty$ norm is minimized .
%It is because each measurement introduces an extra constraint with an additional constraint, as can be seen from~\eqref{eq:nonMinProbQuad}. This on the one hand, allows us to use low degree polynomials and makes the problem formulation flexible. On the other hand, one can also avoid introducing new variables (and hence the constraints) by imposing the conic constraints on the $\sum_i{\text{trace}(\mathsf{G}_i\mathsf{X})^2}$ of~\eqref{eq:nonMinProbQuad} as discussed in Section~\ref{sec:Discussion}, 
%if one desires to improve the speed instead of the flexibility. The new variables are avoided by introducing new conic constraints, our method performs very similarly as the results reported in Figure~\ref{fig:non-minimalRigid}, while offering the speed reported in Figure~\ref{fig:minimalRigid}.
Unlike the  minimal case, \textbf{Olsson'08} performs very similar to \textbf{Olsson-BnB} and \textbf{Briales'17}.

\begin{figure}[h]
\includegraphics[width=4.1cm]{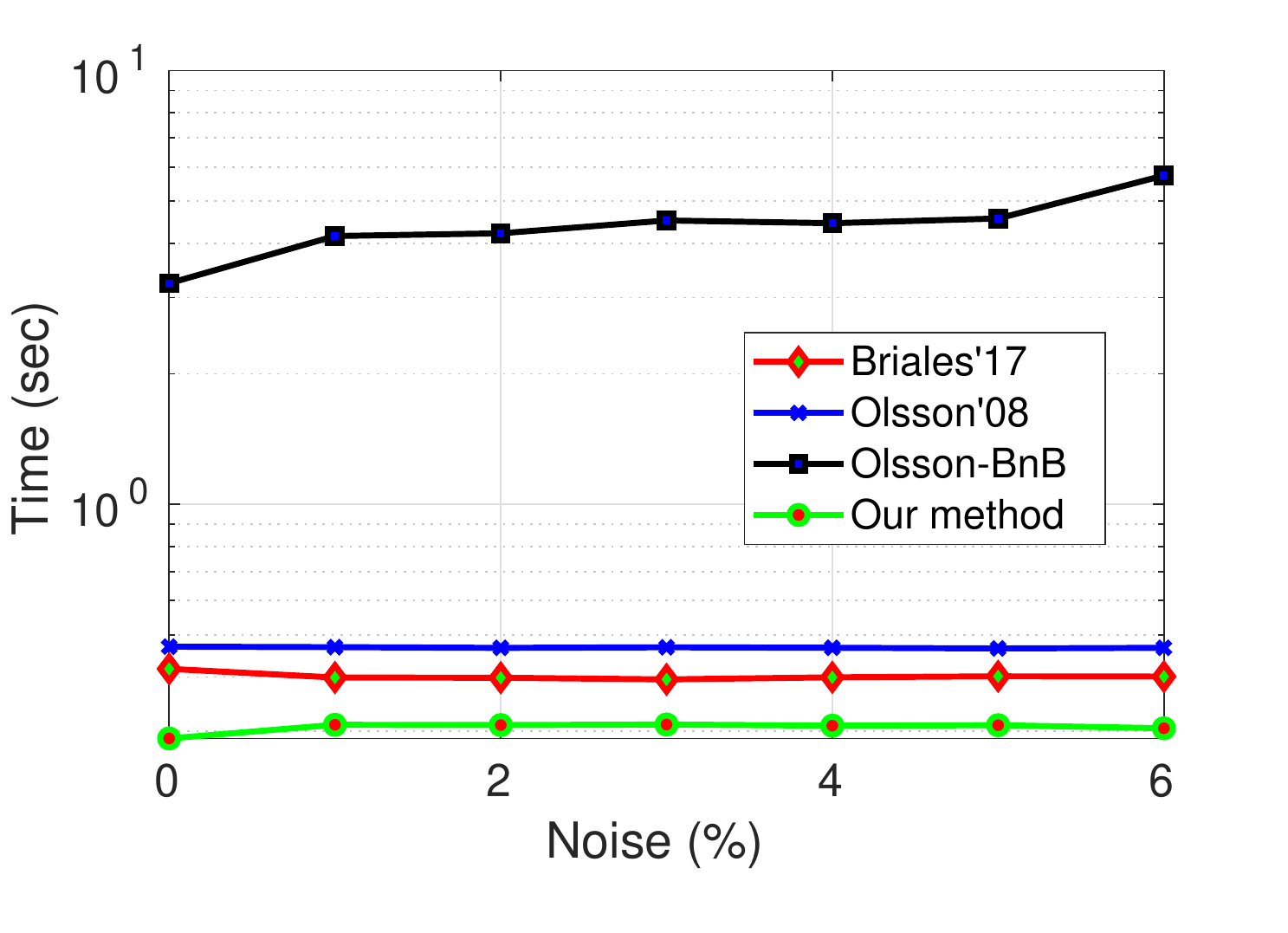}
\includegraphics[width=4.1cm]{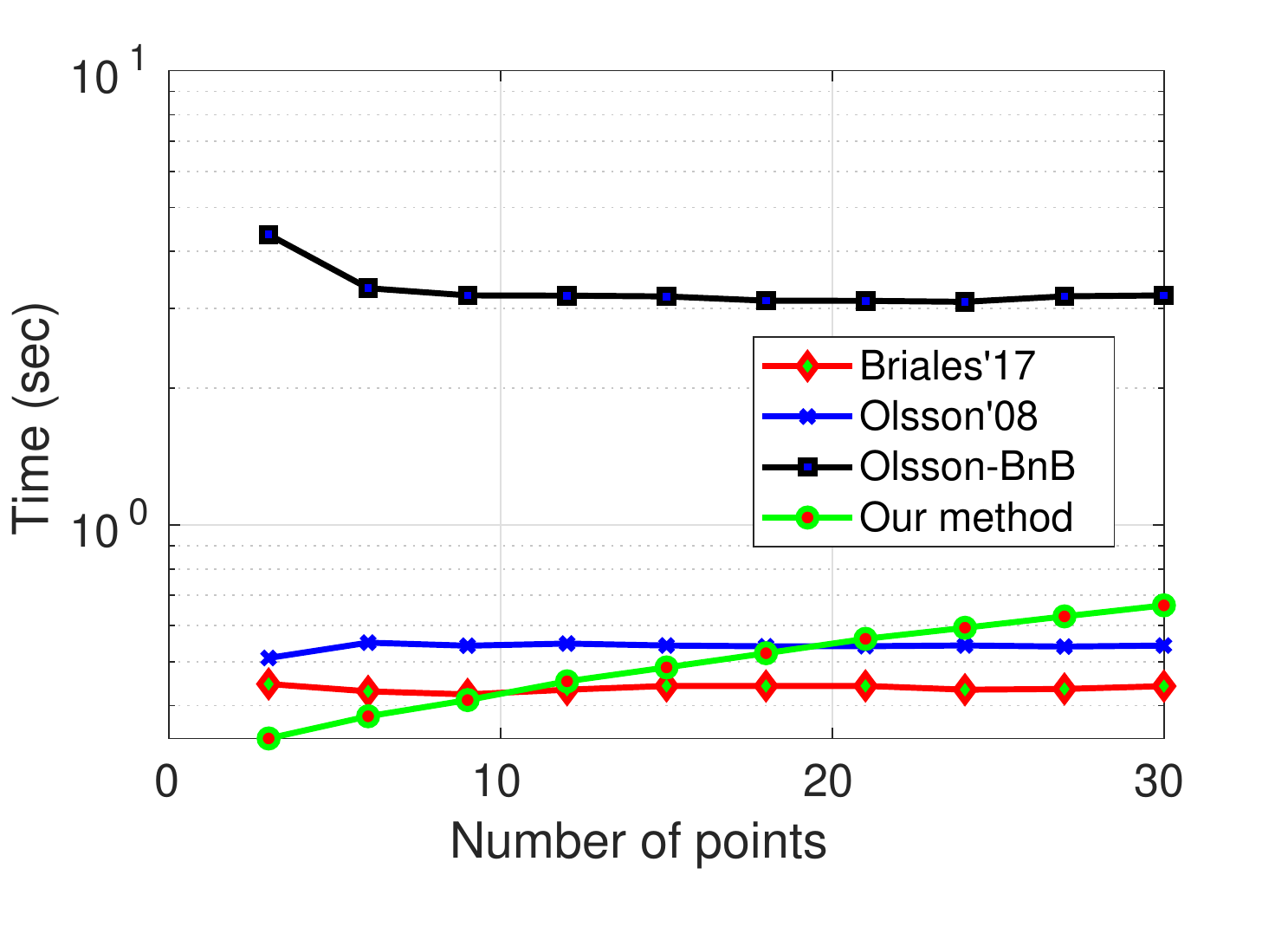}
%\vspace{-2mm}
\caption{Time taken for minimal case (3 points) with noise and non-minimal case (of fixed 1.0\% noise) with number points. Results for rigid body transformation estimation.  \label{fig:Time-non-minimalRigid}}
\end{figure}
%\vspace{-3mm}
\begin{figure}[h]
\includegraphics[width=4.2cm]{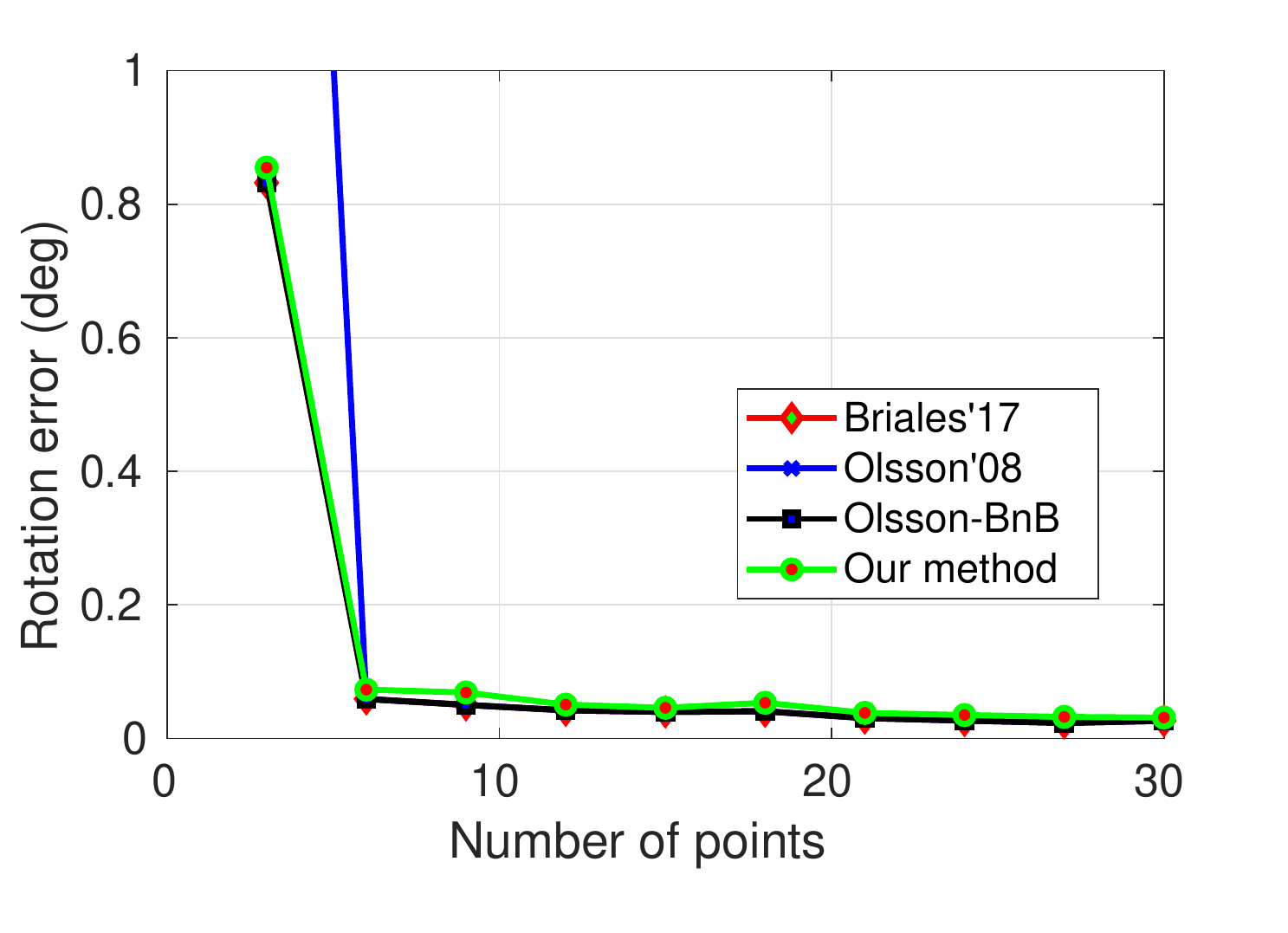}
\includegraphics[width=4.0cm]{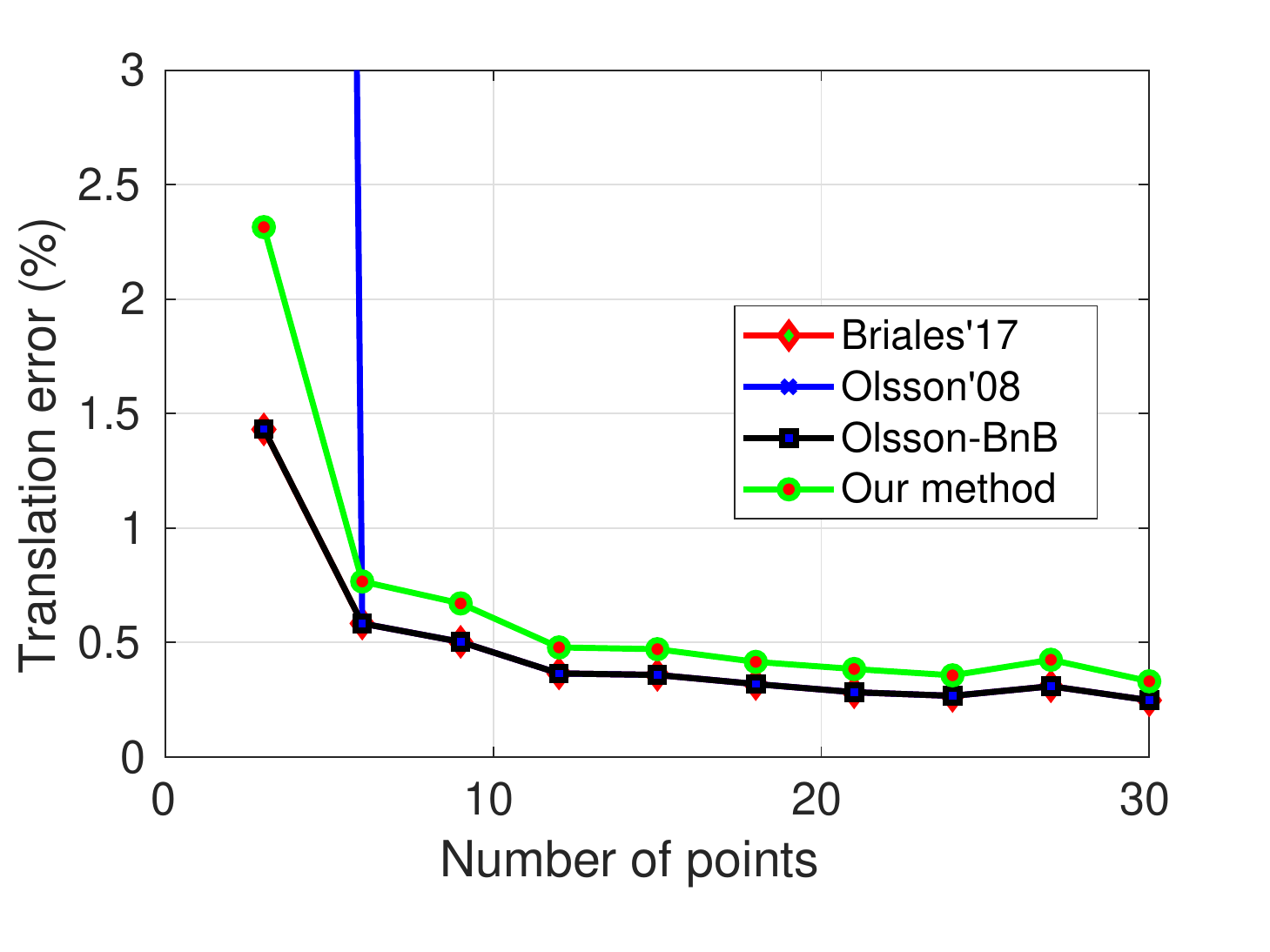}
%\vspace{-2mm}
\caption{Non-minimal (1.0\% noise) results for rigid body transformation estimation. Errors in rotation angles and translation vs. increasing number of points.\label{fig:non-minimalRigid}}
\vspace{-1mm}
\end{figure}

Using the flexibility offered by our method, we conducted experiments with increasing amount of outliers shown in Figure~\ref{fig:utliersNonMinConMax}. Our framework provides consistent results even when outliers are present in the measurements, as it passes through the process of consensus maximization. Although it is unfair to compare with non-minimal solver against a framework that offers both consensus maximization and non-minimal solution, Figure~\ref{fig:utliersNonMinConMax} shows the impact of outliers when only non-minimal solvers are used. Note that our method performs  well, even when the data is contaminated with 90\% outliers (10 inliers and 90 outliers). 
As expected, the non-minimal solvers remain consistent in time with the increase of outliers, of which they are indifferent, as shown in  Figure~\ref{fig:timeOutliersNonMinConMax} (left). On the other hand, our framework filters outliers prior to solving the non-minimal problem. Our method was also compared with a global consensus maximization \textbf{Speciale'17}~\cite{speciale2017consensus}, where both methods perform very similarly in terms of optimality, whereas, \textbf{Speciale'17} performs faster due to its problem specific constraints. The time comparison with increasing number of outliers is shown in  Figure~\ref{fig:timeOutliersNonMinConMax} (right).   

\begin{figure}[h]
\includegraphics[width=4.1cm]{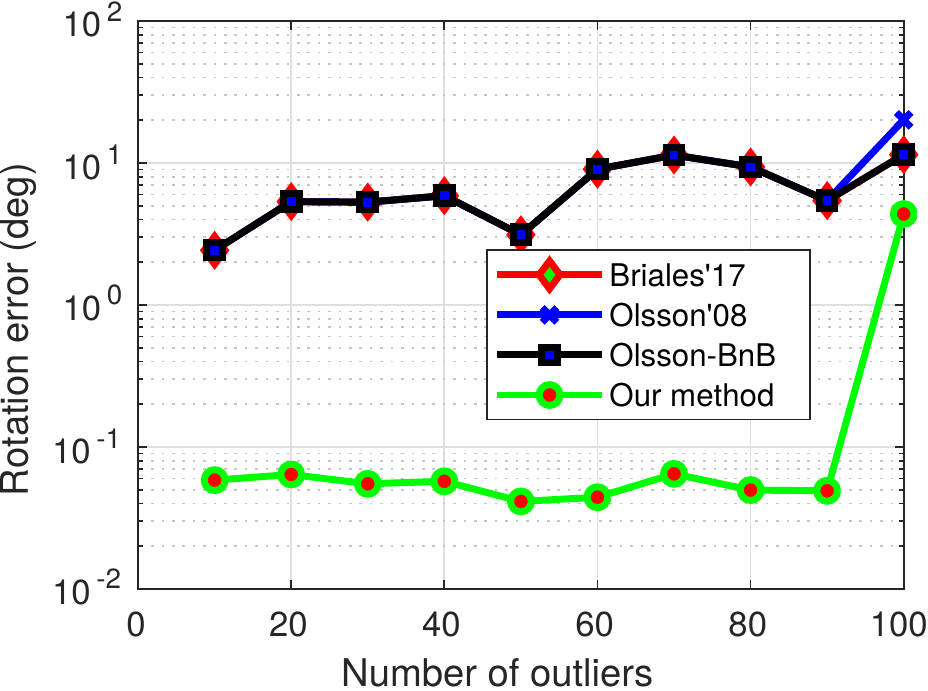}
\includegraphics[width=4.1cm]{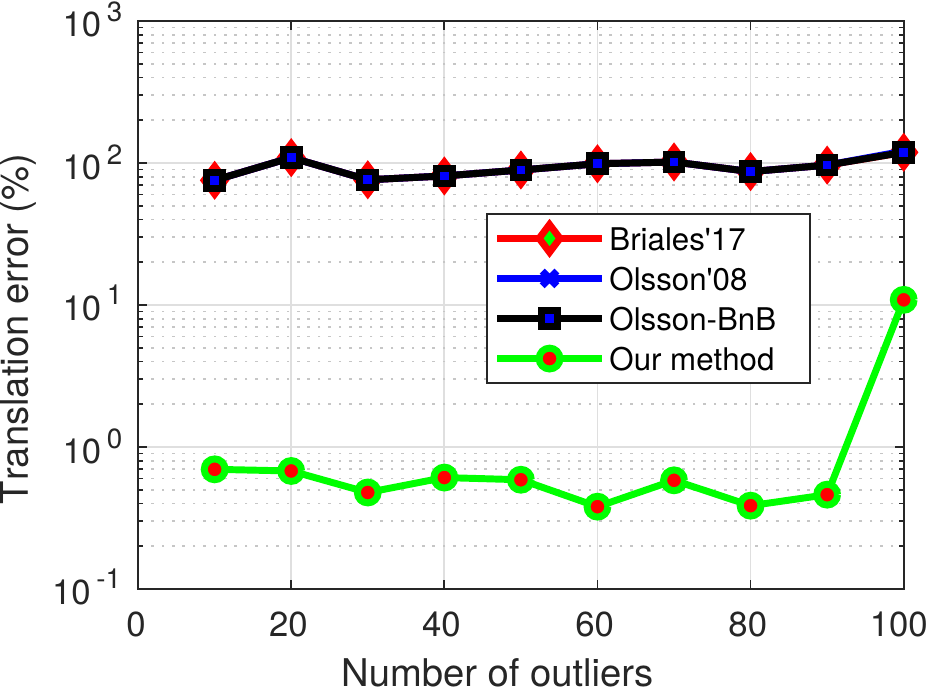}
%\vspace{-3mm}
\caption{Our framework vs. other non-minimal solvers with increasing outliers (and fixed 10 inlier correspondences) for rigid body transformation estimation. Errors in rotation and translation estimation with increasing outliers\label{fig:utliersNonMinConMax}.}
\vspace{-3mm}
\end{figure}

\begin{figure}[h]
\includegraphics[width=4.0cm]{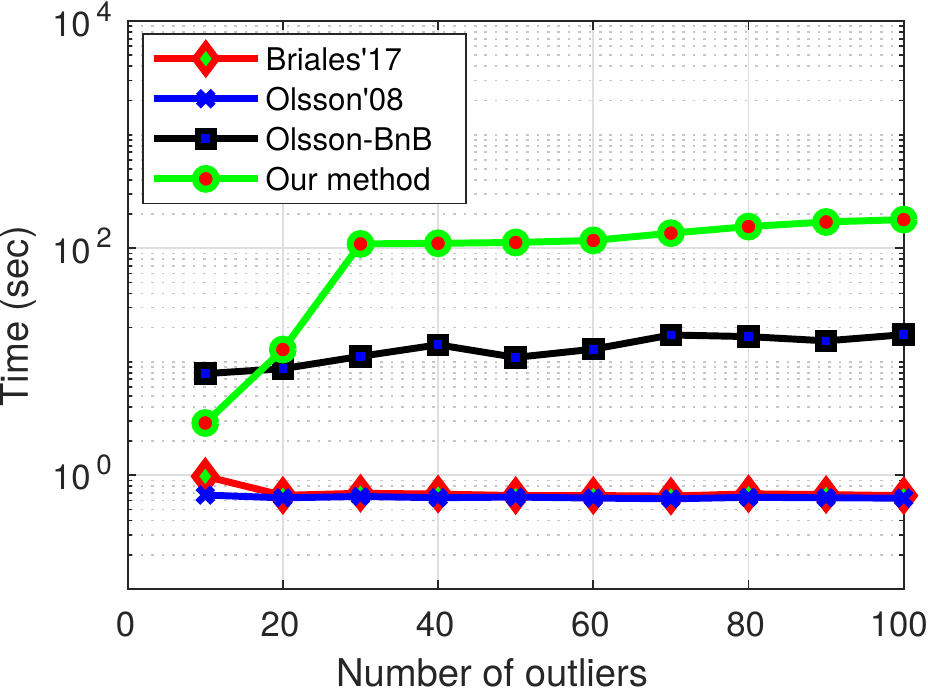}
\includegraphics[width=4.2cm]{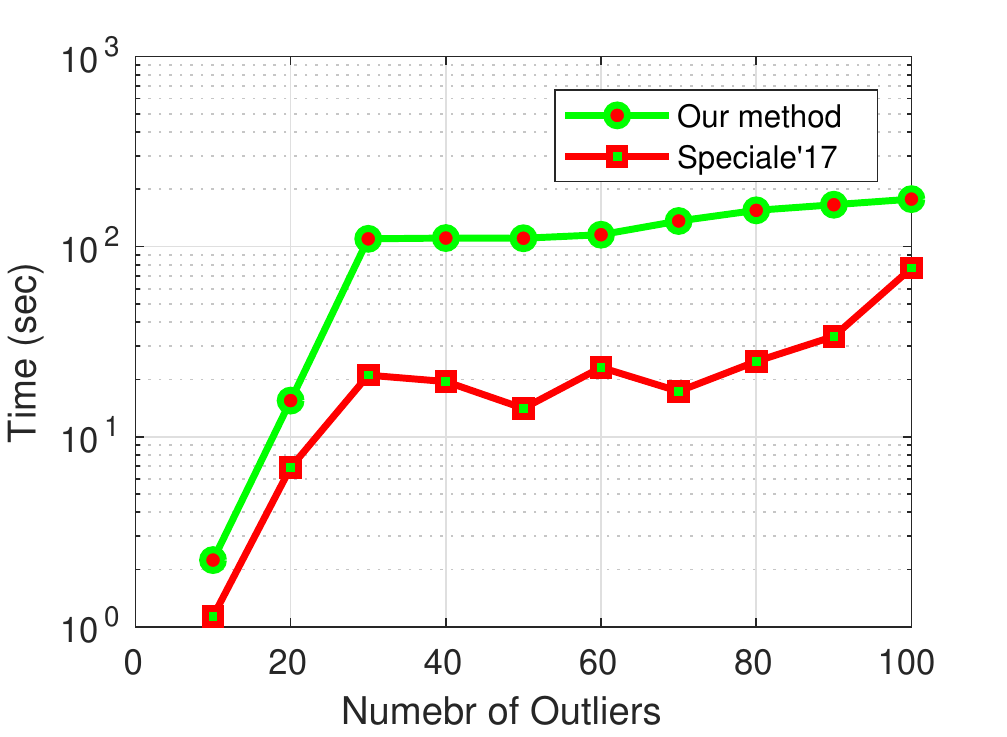}
\vspace{-3mm}
\caption{ Time comparison with global non-minimal solvers (left) and global consensus maximization method (right), with increasing outliers (and fixed 10 inlier correspondences) for rigid body transformation estimation.\label{fig:timeOutliersNonMinConMax}}
\vspace{-3mm}
\end{figure}

\subsection{Camera Autocalibration}
We conducted experiments for camera autocalibration on two real datasets: Fountain and Herz-Jesu from~\cite{Strecha2008}. The results obtained by our framework is compared against that of three  existing global methods for autocalibration: \textbf{LMI direct} method from~\cite{habed2014efficient}, \textbf{Rank-3 direct} method from~\cite{chandraker2007autocalibration},  and a \textbf{Stratified} method from~\cite{chandraker2007globally}. All of these three methods assume that the projective motion required for autocalibration is free of outliers, and minimize the global cost in an optimal manner. Therefore, they can be thought of as non-minimal solvers. The projective reconstruction required for~\cite{habed2014efficient}, ~\cite{chandraker2007autocalibration} and ~\cite{chandraker2007globally} were obtained using~\cite{oliensis2007iterative}. We also compared our method with two local methods for camera calibration, namely~\textbf{Practical} form~\cite{gherardi2010} and \textbf{Simplified-Kruppa} from~\cite{lourakis1999camera}.  We provide the quantitative results for calibration accuracy, by computing errors on the camera intrinsic parameters. Three different error metrics are used: errors in focal length $\Delta f$, principal point $\Delta uv$, skew $\Delta s$. Obtained results by all six methods are reported in Table~\ref{tab:calibCompare}.

 For the consensus maximization, we synthetically introduced the outlier Kruppa's equations. The ourliers are added in an increasing manner up to 80\%. Our consensus maximization method, with Shor's relaxation, is able to detect all inliers and outliers correctly for both datasets. We tracked the number of pessimistic and optimistic inliers for increasing BnB iterations, which is shown in the Figure~\ref{fig:camCalibPlots} (left). The consensus maximization experiments are compared with a global consensus maximization method for autocalibration, namely \textbf{Paudel'18} from~\cite{paudel2018sampling}. For a fair comparison, the constraint on the bounds of the DIAC are chosen as in~\cite{paudel2018sampling}. We therefore assume that the focal length lies within [1 10] interval relative to the image size, aspect ratio lies between 0.7-1.25,  principal point lies within a radius of $(\frac{1}{4})^{th}$ of the image size from the image center, and the skew is close to zero. With these assumptions, we derive bounds on DIAC using the interval analysis arithmetic~\cite{alefeld2000interval}, similarly for the corresponding variable $\mathsf{X}$ in~\eqref{eq:samConMaxQad}.
 Note that these are valid assumptions in most cases for camera calibration. Our results as well that of~\textbf{Paudel'18} are reported in Figure~\ref{fig:camCalibPlots} (right) for increasing outliers, showing the detected inliers and time taken by both methods. These results also show that our framework can greatly benefit if the bounds on the sought parameters are also known. In fact, this may very often be the case in many 3D vision problems.
 
\begin{table}[t]
\scriptsize
\centering
\setlength\tabcolsep{4.5pt}
\begin{tabular}{|c|c|c|c|c|c|c|}
\hline
{Dataset} & Method & $\Delta f$ & $\Delta uv$ & $\Delta s$  & Time(s)\\
\hline
\hline
\multirow{2}{*}{} 
& Practical~\cite{gherardi2010}&0.0117&0.0149&0.0037&0.36\\
\cline{2-6}
& Stratified~\cite{chandraker2007globally} 
&0.0777 & 0.0969 & 0.0125 & 388.24\\
\cline{2-6}
Fountain& Rank-3 Direct~\cite{chandraker2007autocalibration}&0.0100&0.0147&0.0044&5.75\\
\cline{2-6}
(11-views)& LMI Direct~\cite{habed2014efficient} 
&0.0506&0.0269&0.0024&156.88\\ 
\cline{2-6}
& Simplified-Kruppa~\cite{lourakis1999camera} 
& 2.93e-05 & 0.0069 & 3.23e-05  & 1.88 \\
\cline{2-6}
& Ours  
& 0.0060 & 0.0061  &  6.46e-04 &  0.83\\
\hline
\multirow{ 2}{*}{} 
& Practical~\cite{gherardi2010}&0.0017&0.0113&0.0068& 0.36\\
\cline{2-6}
& Stratified~\cite{chandraker2007globally} 
&0.7231 & 0.4462 & 0.3232  & 380.72 \\
\cline{2-6}
Herz-jesu& Rank-3 Direct~\cite{chandraker2007autocalibration}&0.0026&0.0096&0.0069&1 6.54\\
\cline{2-6}
(8-views)& LMI Direct~\cite{habed2014efficient} 
&0.0138&0.0086&0.005&115.61\\
\cline{2-6}
& Simplified-Kruppa~\cite{lourakis1999camera} & 4.46e-05 & 0.0069 & 3.40e-05    & 0.53\\ 
\cline{2-6}
& Ours &  0.0128 & 0.0220  &  8.77e-04 &   0.81 \\
\hline
\end{tabular}
\caption{Our framework vs. two local~\cite{gherardi2010, lourakis1999camera} and three global~\cite{chandraker2007globally,chandraker2007autocalibration,habed2014efficient}  methods in non-minimal setup of camera autocalibration on  two real datasets. Views-1 are the number of measurements used in all the experiments. \label{tab:calibCompare}} 
 \end{table}
 
 \begin{figure}
     \centering
     \includegraphics[width=3.5cm]{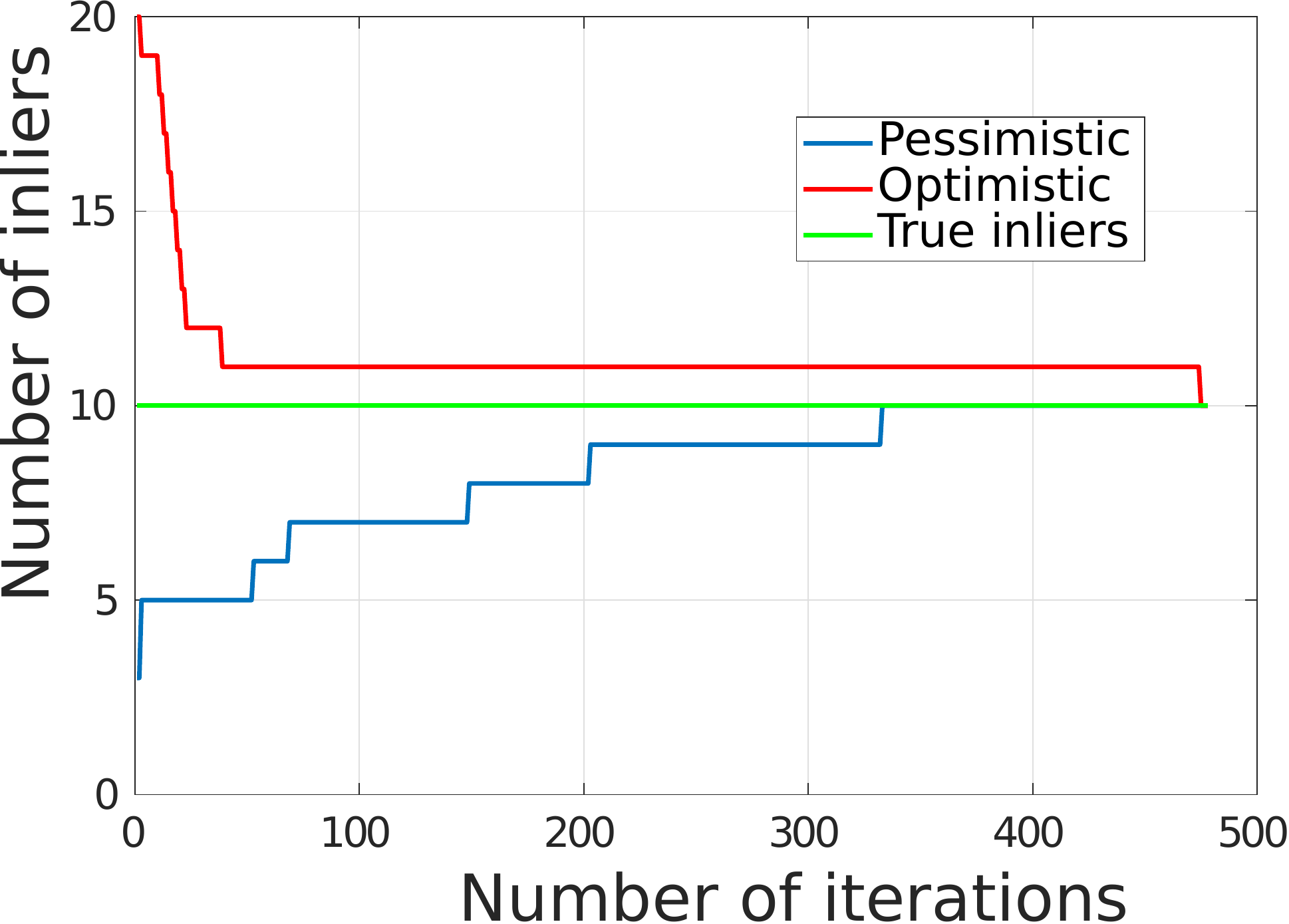}
     \includegraphics[width=4.7cm]{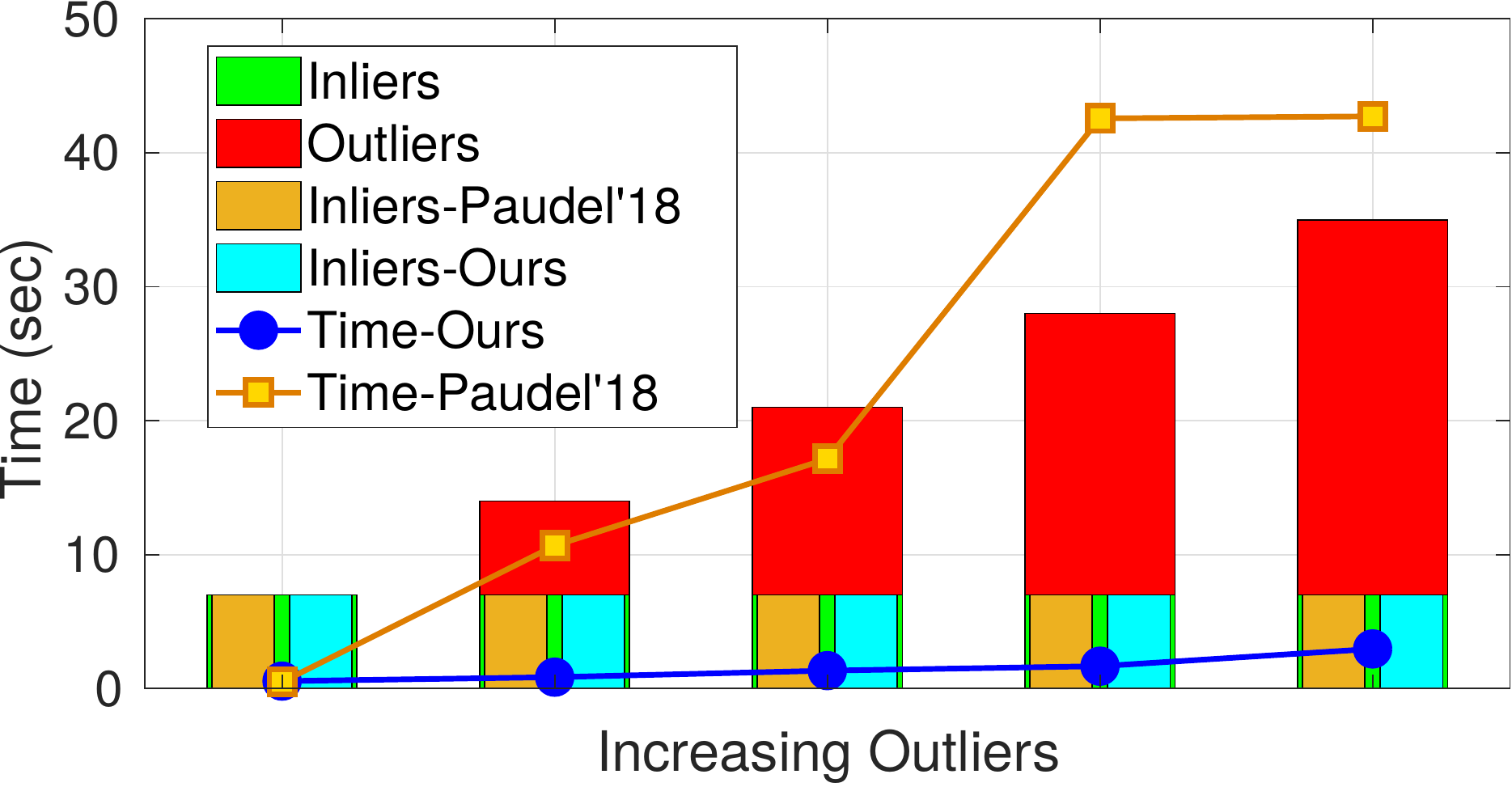}
     \caption{Left: convergence graph for Fountain dataset. Right: time taken and inliers detected by our method and ~\cite{paudel2018sampling}, with increasing outliers and fixed inliers on Herz-jesu.}
     \vspace{-3mm}
     \label{fig:camCalibPlots}
 \end{figure}
 
 \subsection{Non-Rigid Structure-from-Motion}
Experiments with NRSfM were conducted in the similar setup of \textbf{Parashar'18}~\cite{parashar2018isometric}. From $N$ images,  $2N-2$ polynomials of 2 variables on  Christoffel symbols $k_1$ and $k_2$, i.e.~\eqref{eq:iso-nrsfm}, were extracted using the theory presented in Section~\ref{subSec:Nrsfm}. For these polynomials, we first computed the non-minimal solution using the proposed framework with Shor's relaxation. However, the Shor's relaxation alone did not provide us satisfactory solutions.  Note that~\eqref{eq:iso-nrsfm} are expressed using the second-order measurements of the local correspondences, therefore their coefficients are very sensitive to noise. Moreover, polynomials of~\eqref{eq:iso-nrsfm} are solved for each point independently. In such cases, non-minimal solvers are very valuable, where the over-determined system of quartic polynomials are provided by multiple measurements across views.
Therefore, the method of~\cite{parashar2018isometric} uses the hierarchy of SOS relaxations to obtain the desired solution. In this context, we observed that the  Lasserre's relaxation is indeed necessary, because the non-minimal solutions obtained using Shor's relaxation were not on par with that of SOS method. Therefore, we show the difference between Shor's and Lasserre's relaxations for NRSfM. In previous two problems, their differences were not very significant in terms of accuracy, but were so in speed.

We use the datasets, Flag~\cite{White2007}, Hulk and
Tshirt~\cite{parashar2016isometric} to evaluate our non-minimal solver. Our obtained solutions are compared with the baseline \textbf{Parashar'18}~\cite{parashar2018isometric} in  Figure~\ref{fig:nonRigidNonMin} for the non-minimal case, where the reconstructed depth error is shown across views on the left, and across datasets on the right. Note from Figure~\ref{fig:nonRigidNonMin} that the baseline method performs significantly better compared to only the Shor's relaxations-based method. However, when we use Lasserre's relaxation of order one (i.e. $s=1$ in~\eqref{eq:relaxLasser}), the solution becomes very competitive to  that of the hierarchy of SOS relaxations used in~\cite{parashar2018isometric}. This shows that the higher order relaxations do not offer significant improvements at least for this particular formulation of NRSfM.
In fact, it is generally agreed that NRSfM is a very challenging problem in 3D vision. Our observation from real dataset shows that higher order relaxations may not be necessary, in a wide range of applications, including the case of isometric NRSfM. Of course, our observation may be biased from only three problems that we have tested. Especially, second degree polynomials for rigid body transformation and camera calibraion, and degree four polynomials that involve only two unknown variables at a time in the case of NRSfM. Nevertheless, one needs to be aware that lower order relaxations could be tried first,  before extending into the higher ones. If the lower order relaxations already offer satisfactory solutions, this not only allows one to obtain non-minimal solutions faster but also allows to use rather non obvious polynomial problems within the global framework of consensus maximization by using the BnB paradigm.  

\begin{figure}
    \centering
  \includegraphics[width=4.1cm]{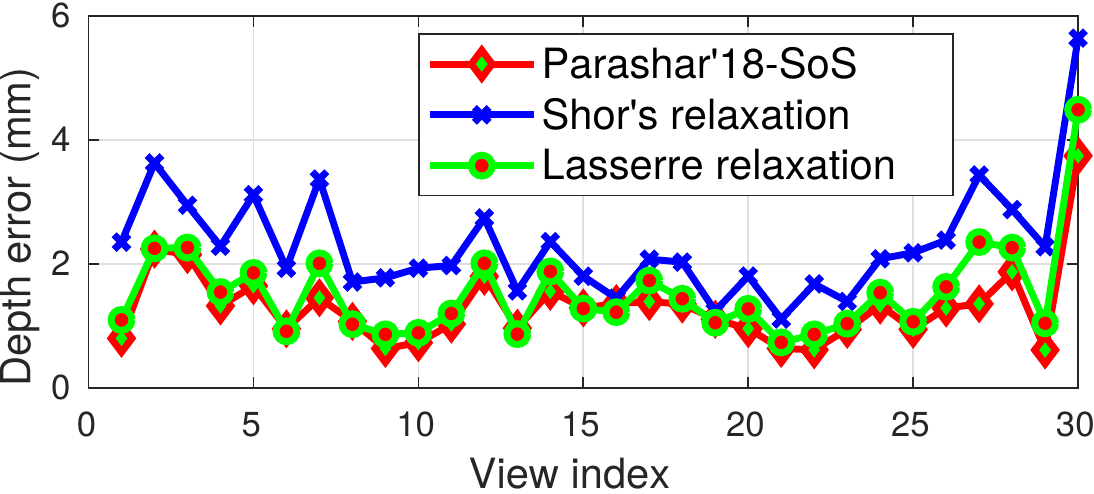}
\includegraphics[width=4.1cm]{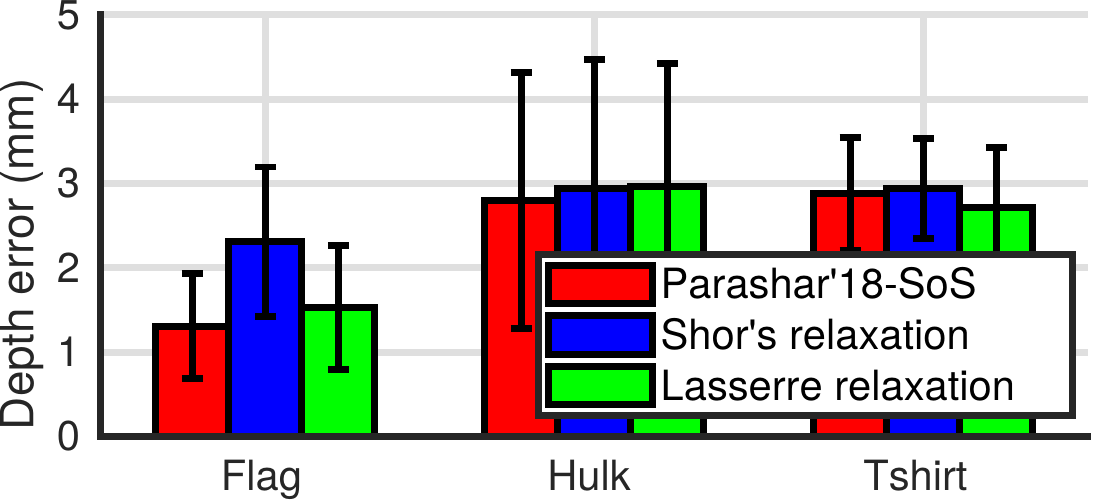}
    \caption{Results for non-minimal isometric non-rigid reconstruction. Left to right: depth and normal errors  for flag dataset, depth errors of three methods on different datates. \label{fig:nonRigidNonMin}}
    \vspace{-3mm}
\end{figure}

To support our claim that even the lower degree relaxations are sufficient for consensus maximization, we conducted several experiments with various amount of synthetic outliers on aforementioned real datasets. For the setup of NRSfM we were able to detect almost all outliers when outliers up to 70\% were introduced. Results for one such instance of consensus maximization, with 50\% outliers views, are reported in Figure~\ref{fig:nonRigidConvMax} on the right and the estimated depth error across inlier views on the left.  As expected, results with Lasserre's relaxations are significantly better than that of the Shor's relaxation. Nevertheless, Shor's relaxation still shows its expected behaviour. 
Although, it may not be very interesting to compare the results of a non-minimal solver, i.e. \textbf{Parashar'18}, against that of a consensus maximization method in general, it is somehow different in this case. One aspect that we have not yet discussed is the power of SOS methods. It is generally known that the SOS solvers are robust to noise~\cite{parrilo2000structured}, SOS solver appears to be relatively stable even in the presence of outliers. In the first glance, we thought this could be because of the iterative refinement of outliers within the algorithm of~\textbf{Parashar'18}. This turns out not to be the case. The SOS method consistently performed reasonably well, even when the iterative refinement process was removed. This is because of two reasons: one is the robustness of SOS methods towards a moderate amount of outliers; other is the failure to obtain a valid solution means in some sense the detection of outliers. Such outlier detection  is specific to the problem formulation of~\cite{parashar2018isometric} because the POP in this case is defined point-wise. 

\begin{figure}[h]
\includegraphics[width=4.1cm]{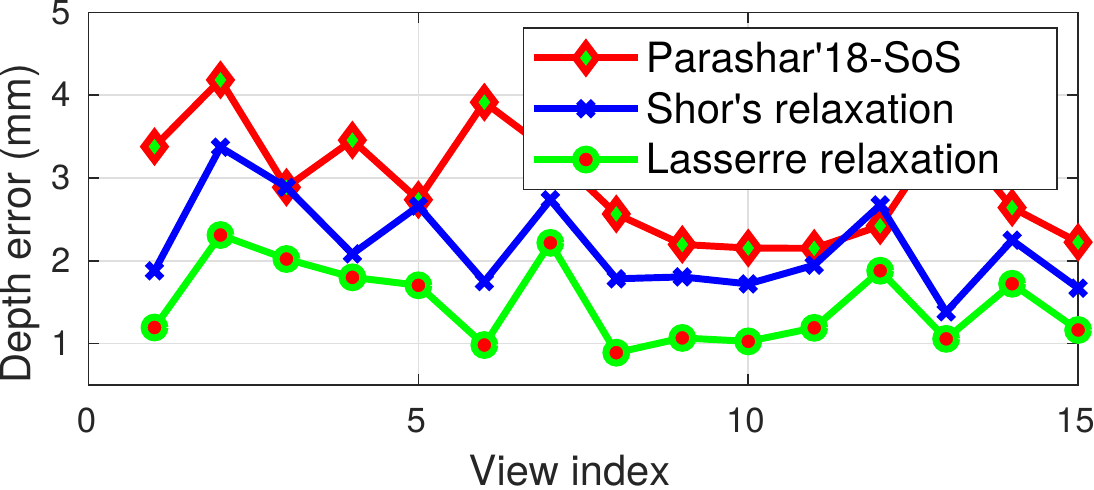}
\includegraphics[width=4.1cm]{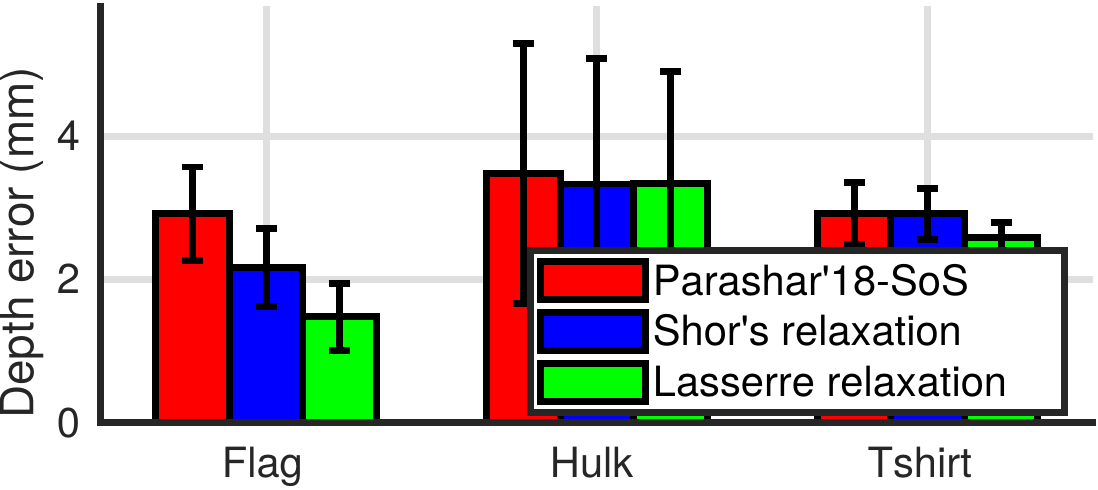}
\caption{Consensus maximization results for isometric non-rigid reconstruction. Left to right: depth and normal errors  for flag dataset, depth errors of three methods on different datasets with 50\% synthetically added outliers.\label{fig:nonRigidConvMax}}
\vspace{-3mm}
\end{figure}

\section{Conclusion}
In this paper, we demonstrated that a proper usage of the existing tools in numerical algebraic geometric POP can be used in a straightforward way in many 3D vision problems. This is achieved by using the known ``good approximate solutions" expressed as convex SDP formulation. On the one hand, we argue that the existing solutions can be used in their current form to solve many 3D vision problems, especially for which the optimal non-minimal solvers have not been devised yet. We also discussed about a good practice to formulate POP. Using our theoretical reasoning, we made suggestions for a generic relaxed non-minimal solver that is suitable for may 3D vision problems. We further argued that the standard method for polynomial relaxations are indeed powerful, which can also be used for consensus maximization in a deterministic manner. We have supported our suggestions/claims using several experiments of three diverse problems in 3D vision. We reach to  this conclusion mainly because many polynomials in 3D vision problems are inherently of low degrees with limited variables. 

\paragraph*{Acknowledgements.}
This research has received funding from the EU Horizon 2020 research and innovation programme under grant agreement No.\ 820434. The research was also funded by the ETH Zurich project with SPECTA.

{\small
\bibliographystyle{ieee_fullname}
\bibliography{egbib}

\begin{thebibliography}{10}\itemsep=-1pt

\bibitem{ahmadi2016geometry}
Amir~Ali Ahmadi, Georgina Hall, Ameesh Makadia, and Vikas Sindhwani.
\newblock Geometry of 3d environments and sum of squares polynomials.
\newblock {\em arXiv preprint arXiv:1611.07369}, 2016.

\bibitem{alefeld2000interval}
G{\"o}tz Alefeld and G{\"u}nter Mayer.
\newblock Interval analysis: theory and applications.
\newblock {\em Journal of computational and applied mathematics},
  121(1):421--464, 2000.

\bibitem{barath2017minimal}
Daniel Barath, Tekla Toth, and Levente Hajder.
\newblock A minimal solution for two-view focal-length estimation using two
  affine correspondences.
\newblock 2017.

\bibitem{bartoli2017generalizing}
Adrien Bartoli.
\newblock Generalizing the prediction sum of squares statistic and formula,
  application to linear fractional image warp and surface fitting.
\newblock {\em International Journal of Computer Vision}, 122(1):61--83, 2017.

\bibitem{bazin2014globally}
Jean-Charles Bazin, Yongduek Seo, Richard Hartley, and Marc Pollefeys.
\newblock Globally optimal inlier set maximization with unknown rotation and
  focal length.
\newblock In {\em European Conference on Computer Vision}, pages 803--817.
  Springer, 2014.

\bibitem{boyd1997semidefinite}
Stephen Boyd and Lieven Vandenberghe.
\newblock Semidefinite programming relaxations of non-convex problems in
  control and combinatorial optimization.
\newblock In {\em Communications, Computation, Control, and Signal Processing},
  pages 279--287. Springer, 1997.

\bibitem{boyd}
Stephen Boyd and Lieven Vandenberghe.
\newblock {\em Convex Optimization}.
\newblock Cambridge University Press, New York, NY, USA, 2004.

\bibitem{briales2017convex}
Jesus Briales, Javier Gonzalez-Jimenez, et~al.
\newblock Convex global 3d registration with lagrangian duality.
\newblock In {\em International Conference on Computer Vision and Pattern
  Recognition (CVPR)}, 2017.

\bibitem{briales2018certifiably}
Jesus Briales, Laurent Kneip, SIST ShanghaiTech, and Javier Gonzalez-Jimenez.
\newblock A certifiably globally optimal solution to the non-minimal relative
  pose problem.
\newblock In {\em Proceedings of the IEEE Conference on Computer Vision and
  Pattern Recognition}, pages 145--154, 2018.

\bibitem{cai2018deterministic}
Zhipeng Cai, Tat-Jun Chin, Huu Le, and David Suter.
\newblock Deterministic consensus maximization with biconvex programming.
\newblock In {\em European Conference on Computer Vision}, 2018.

\bibitem{Chandraker:2007}
Manmohan Chandraker, Sameer Agarwal, Fredrik Kahl, David Nister, and David
  Kriegman.
\newblock Autocalibration via rank-constrained estimation of the absolute
  quadric.
\newblock In {\em IEEE Conference on Computer Vision and Pattern Recognition
  (CVPR)}, 2007.

\bibitem{chandraker2007autocalibration}
Manmohan Chandraker, Sameer Agarwal, Fredrik Kahl, David Nist{\'e}r, and David
  Kriegman.
\newblock Autocalibration via rank-constrained estimation of the absolute
  quadric.
\newblock In {\em Computer Vision and Pattern Recognition, 2007. CVPR'07. IEEE
  Conference on}, pages 1--8. IEEE, 2007.

\bibitem{chandraker2007globally}
Manmohan Chandraker, Sameer Agarwal, David Kriegman, and Serge Belongie.
\newblock Globally optimal affine and metric upgrades in stratified
  autocalibration.
\newblock In {\em Computer Vision, 2007. ICCV 2007. IEEE 11th International
  Conference on}, pages 1--8. IEEE, 2007.

\bibitem{chin2018robust}
Tat-Jun Chin, Zhipeng Cai, and Frank Neumann.
\newblock Robust fitting in computer vision: Easy or hard?
\newblock In {\em European Conference on Computer Vision}, 2018.

\bibitem{chin2016guaranteed}
Tat-Jun Chin, Yang Heng~Kee, Anders Eriksson, and Frank Neumann.
\newblock Guaranteed outlier removal with mixed integer linear programs.
\newblock In {\em Proceedings of the IEEE Conference on Computer Vision and
  Pattern Recognition}, pages 5858--5866, 2016.

\bibitem{chin2015efficient}
Tat-Jun Chin, Pulak Purkait, Anders Eriksson, and David Suter.
\newblock Efficient globally optimal consensus maximisation with tree search.
\newblock In {\em Proceedings of the IEEE Conference on Computer Vision and
  Pattern Recognition}, pages 2413--2421, 2015.

\bibitem{chinneck2007feasibility}
John~W Chinneck.
\newblock {\em Feasibility and Infeasibility in Optimization:: Algorithms and
  Computational Methods}, volume 118.
\newblock Springer Science \& Business Media, 2007.

\bibitem{ErikssonOKC18}
Anders~P. Eriksson, Carl Olsson, Fredrik Kahl, and Tat{-}Jun Chin.
\newblock Rotation averaging and strong duality.
\newblock In {\em CVPR}, 2018.

\bibitem{gherardi2010}
Riccardo Gherardi and Andrea Fusiello.
\newblock Practical autocalibration.
\newblock {\em Computer Vision--ECCV 2010}, pages 790--801, 2010.

\bibitem{habed2014efficient}
Adlane Habed, Danda Pani~Paudel, C{\'e}dric Demonceaux, and David Fofi.
\newblock Efficient pruning lmi conditions for branch-and-prune rank and
  chirality-constrained estimation of the dual absolute quadric.
\newblock In {\em Proceedings of the IEEE Conference on Computer Vision and
  Pattern Recognition}, pages 493--500, 2014.

\bibitem{kahl2007globally}
Fredrik Kahl and Didier Henrion.
\newblock Globally optimal estimates for geometric reconstruction problems.
\newblock {\em International Journal of Computer Vision}, 74(1):3--15, 2007.

\bibitem{kneip2014upnp}
Laurent Kneip, Hongdong Li, and Yongduek Seo.
\newblock Upnp: An optimal o (n) solution to the absolute pose problem with
  universal applicability.
\newblock In {\em European Conference on Computer Vision}, pages 127--142.
  Springer, 2014.

\bibitem{LajoieHBC19}
Pierre{-}Yves Lajoie, Siyi Hu, Giovanni Beltrame, and Luca Carlone.
\newblock Modeling perceptual aliasing in {SLAM} via discrete-continuous
  graphical models.
\newblock {\em {IEEE} Robotics and Automation Letters}, 4(2):1232--1239, 2019.

\bibitem{lasserre2000convergent}
Jean~B Lasserre.
\newblock Convergent lmi relaxations for nonconvex quadratic programs.
\newblock In {\em Decision and Control, 2000. Proceedings of the 39th IEEE
  Conference on}, volume~5, pages 5041--5046. IEEE, 2000.

\bibitem{lasserre2001global}
Jean~B Lasserre.
\newblock Global optimization with polynomials and the problem of moments.
\newblock {\em SIAM Journal on optimization}, 11(3):796--817, 2001.

\bibitem{lasserre2002semidefinite}
Jean~B Lasserre.
\newblock Semidefinite programming vs. lp relaxations for polynomial
  programming.
\newblock {\em Mathematics of operations research}, 27(2):347--360, 2002.

\bibitem{li2009consensus}
Hongdong Li.
\newblock Consensus set maximization with guaranteed global optimality for
  robust geometry estimation.
\newblock In {\em Computer Vision, 2009 IEEE 12th International Conference on},
  pages 1074--1080. IEEE, 2009.

\bibitem{lourakis1999camera}
Manolis~IA Lourakis and Rachid Deriche.
\newblock {\em Camera self-calibration using the singular value decomposition
  of the fundamental matrix: From point correspondences to 3D measurements}.
\newblock PhD thesis, INRIA, 1999.

\bibitem{mosek2012mosek}
ApS MOSEK.
\newblock The mosek optimization toolbox for matlab manual, version 8.0.
\newblock {\em MOSEK ApS, Denmark}, 2015.

\bibitem{oliensis2007iterative}
John Oliensis and Richard Hartley.
\newblock Iterative extensions of the sturm/triggs algorithm: Convergence and
  nonconvergence.
\newblock {\em IEEE Transactions on Pattern Analysis and Machine Intelligence},
  29(12):2217--2233, 2007.

\bibitem{olsson2008solving}
Carl Olsson and Anders Eriksson.
\newblock Solving quadratically constrained geometrical problems using
  lagrangian duality.
\newblock In {\em Pattern Recognition, 2008. ICPR 2008. 19th International
  Conference on}, pages 1--5. IEEE, 2008.

\bibitem{olsson2009branch}
Carl Olsson, Fredrik Kahl, and Magnus Oskarsson.
\newblock Branch-and-bound methods for euclidean registration problems.
\newblock {\em IEEE Transactions on Pattern Analysis and Machine Intelligence},
  31(5):783--794, 2009.

\bibitem{pani2015robust}
Danda Pani~Paudel, Adlane Habed, C{\'e}dric Demonceaux, and Pascal Vasseur.
\newblock Robust and optimal sum-of-squares-based point-to-plane registration
  of image sets and structured scenes.
\newblock In {\em Proceedings of the IEEE International Conference on Computer
  Vision}, pages 2048--2056, 2015.

\bibitem{parashar2016isometric}
Shaifali Parashar, Daniel Pizarro, and Adrien Bartoli.
\newblock Isometric non-rigid shape-from-motion in linear time.
\newblock In {\em Proceedings of the IEEE Conference on Computer Vision and
  Pattern Recognition}, pages 4679--4687, 2016.

\bibitem{parashar2018isometric}
Shaifali Parashar, Daniel Pizarro, and Adrien Bartoli.
\newblock Isometric non-rigid shape-from-motion with riemannian geometry solved
  in linear time.
\newblock {\em IEEE transactions on pattern analysis and machine intelligence},
  40(10):2442--2454, 2018.

\bibitem{parrilo2000structured}
Pablo~A Parrilo.
\newblock {\em Structured semidefinite programs and semialgebraic geometry
  methods in robustness and optimization}.
\newblock PhD thesis, California Institute of Technology, 2000.

\bibitem{paudel2018sampling}
Danda~Pani Paudel and Luc Van~Gool.
\newblock Sampling algebraic varieties for robust camera autocalibration.
\newblock In {\em European Conference on Computer Vision}, pages 275--292.
  Springer, Cham, 2018.

\bibitem{Powers1998}
Victoria Powers and Thorsten W{\"o}rmann.
\newblock An algorithm for sums of squares of real polynomials.
\newblock 1998.

\bibitem{RosenCBL19}
David~M. Rosen, Luca Carlone, Afonso~S. Bandeira, and John~J. Leonard.
\newblock Se-sync: {A} certifiably correct algorithm for synchronization over
  the special euclidean group.
\newblock {\em I. J. Robotics Res.}, 38(2-3), 2019.

\bibitem{shor1987quadratic}
Naum~Z Shor.
\newblock Quadratic optimization problems.
\newblock {\em Soviet Journal of Computer and Systems Sciences}, 25:1--11,
  1987.

\bibitem{shor1987}
Naum~Z Shor.
\newblock Quadratic optimization problems.
\newblock {\em Soviet Journal of Circuits and Systems Sciences}, 25(6):1--11,
  1987.

\bibitem{speciale2017consensus}
Pablo Speciale, Danda~Pani Paudel, Martin~R Oswald, Till Kroeger, Luc Van~Gool,
  and Marc Pollefeys.
\newblock Consensus maximization with linear matrix inequality constraints.
\newblock In {\em Computer Vision and Pattern Recognition (CVPR), 2017 IEEE
  Conference on}, pages 5048--5056. IEEE, 2017.

\bibitem{speciale2018consensus}
Pablo Speciale, Danda~P Paudel, Martin~R Oswald, Hayko Riemenschneider, Luc~V
  Gool, and Marc Pollefeys.
\newblock Consensus maximization for semantic region correspondences.
\newblock In {\em Proceedings of the IEEE Conference on Computer Vision and
  Pattern Recognition}, pages 7317--7326, 2018.

\bibitem{Strecha2008}
C. Strecha, W. von Hansen, L. Van~Gool, P. Fua, and U. Thoennessen.
\newblock On benchmarking camera calibration and multi-view stereo for high
  resolution imagery.
\newblock In {\em IEEE Conference on Computer Vision and Pattern Recognition
  (CVPR)}, pages 1--8, 2008.

\bibitem{White2007}
Ryan White, Keenan Crane, and David~A. Forsyth.
\newblock Capturing and animating occluded cloth.
\newblock In {\em SIGGRAPH 2007}, 2007.

\bibitem{yang2014optimal}
Jiaolong Yang, Hongdong Li, and Yunde Jia.
\newblock Optimal essential matrix estimation via inlier-set maximization.
\newblock In {\em European Conference on Computer Vision}, pages 111--126.
  Springer, 2014.

\end{thebibliography}
}

\end{document}

%%%%%%%%% TITLE
%\title{ Robust Matching for Two Views with Unknown Camera Model and Motion}
%\title{Camera Model and Motion Discovery from Feature Correspondences}
\title{Supplimentary material for ``Convex Relaxations for Consensus and Non-Minimal Problems in 3D  Vision"}
\author{Thomas Probst} 
\author{Danda Pani Paudel}
\author{Ajad Chhatkuli}
\author{Luc Van Gool}
\affil{Computer Vision Laboratory, ETH Zurich, Switzerland}

%% Five different problems will be solved in this manner!
\maketitle
\thispagestyle{empty}
%%%%%%%%% BODY TEXT

This document provides the proof of the Propositions, along with some qualitative results obtained using the proposed framework. In the following, we first provide an alternative proof to propositions~\textcolor{red}{3.2} and ~\textcolor{red}{3.3}. Later, we provide the qualitative results of non-rigid reconstruction, camera autocalibration, and rigid 3D motion estimation, non-rigid reconstruction.

%-- Will be written later --

\begin{proof}[Proof of Proposition~\textcolor{red}{3.2/3.3}:]
Let us define $\epsilon_i=\delta_i^2$, $\mathsf{v}=[\mathsf{x}^\intercal, \delta_1,\ldots,\delta_m ]^\intercal$,   $\mathsf{\Lambda}_0={\scriptsize\begin{bmatrix} 0&\\&\mathsf{I}_{m\times m}\end{bmatrix}}$,  $\mathsf{\Lambda}^+_i={\scriptsize\begin{bmatrix} \mathsf{G}_i&\\&\mathsf{E}^{i}\end{bmatrix}}$ and 
$\mathsf{\Lambda}^-_i={\scriptsize\begin{bmatrix} -\mathsf{G}_i&\\&\mathsf{E}^{i}\end{bmatrix}}$
with $e^{i}_{ii}=1$ and $(e^{i}_{ij})_{j\not=i}=0$. For $g_i(\mathsf{v})=z_1(\mathsf{v})^\intercal\mathsf{\Lambda}_iz_1(\mathsf{v})$,
the POP of polynomials $\{g_0(\mathsf{v}),g^+_i(\mathsf{v}),g^-_i(\mathsf{v}), i=0,\ldots m\}$
can be relaxed, similar to~(\textcolor{red}{2}), as       
\begin{equation}
\underset{\mathsf{\Upsilon}\succeq0}{\text{min}\,\,} \Big\{\text{tr}(\mathsf{\Lambda}_0\mathsf{\Upsilon})|\,\,  \text{tr}(\mathsf{\Lambda}^+_i\mathsf{\Upsilon}),
\text{tr}(\mathsf{\Lambda}^-_i\mathsf{\Upsilon})\geq 0,i=1,\ldots,m \Big\},
\label{eq:shorrlxProof}
\end{equation}
where $\text{tr}(\mathsf{\Lambda}_0\mathsf{\Upsilon}):=g_0(\mathsf{v})=\sum_i{\epsilon_i}$. We choose   $\mathsf{\Upsilon}={\scriptsize\begin{bmatrix} \mathsf{Y}&\\&\mathsf{I}_{m\times m}\end{bmatrix}}$ such that
$g_0(\mathsf{v})=\sum_i{\delta^2_i}=\sum_i{\epsilon_i}$ is satisfied.
For polynomials
$g^+_i(\mathsf{v})=z_1(\mathsf{v})^\intercal\mathsf{\Lambda}^+_iz_1(\mathsf{v})=tr(\mathsf{G}_i\mathsf{Y})-\delta^2_i$ and
$g^-_i(\mathsf{v})=z_1(\mathsf{v})^\intercal\mathsf{\Lambda}^-_iz_1(\mathsf{v})=\delta^2_i-tr(\mathsf{G}_i\mathsf{Y})$, \eqref{eq:shorrlxProof} becomes,
\begin{equation}
\underset{\mathsf{\mathsf{Y}}\succeq0}{\text{min}\,\,} \Big\{\sum_i{\delta^2_i}|\,\,  
-{\delta^2_i}\leq\text{tr}(\mathsf{G}_i\mathsf{Y})\geq {\delta^2_i},i=1,\ldots,m \Big\}.
\label{eq:shorrlxProof2}
\end{equation}
After considering $\epsilon_i=\delta_i^2$,~\eqref{eq:shorrlxProof2} becomes an SDP of~(\textcolor{red}{6/7}).
\hfill \qed
\end{proof}

{\small
\bibliographystyle{ieee}
\bibliography{egbib}
}

\begin{figure*}
    \centering
    \includegraphics[width=5.5cm]{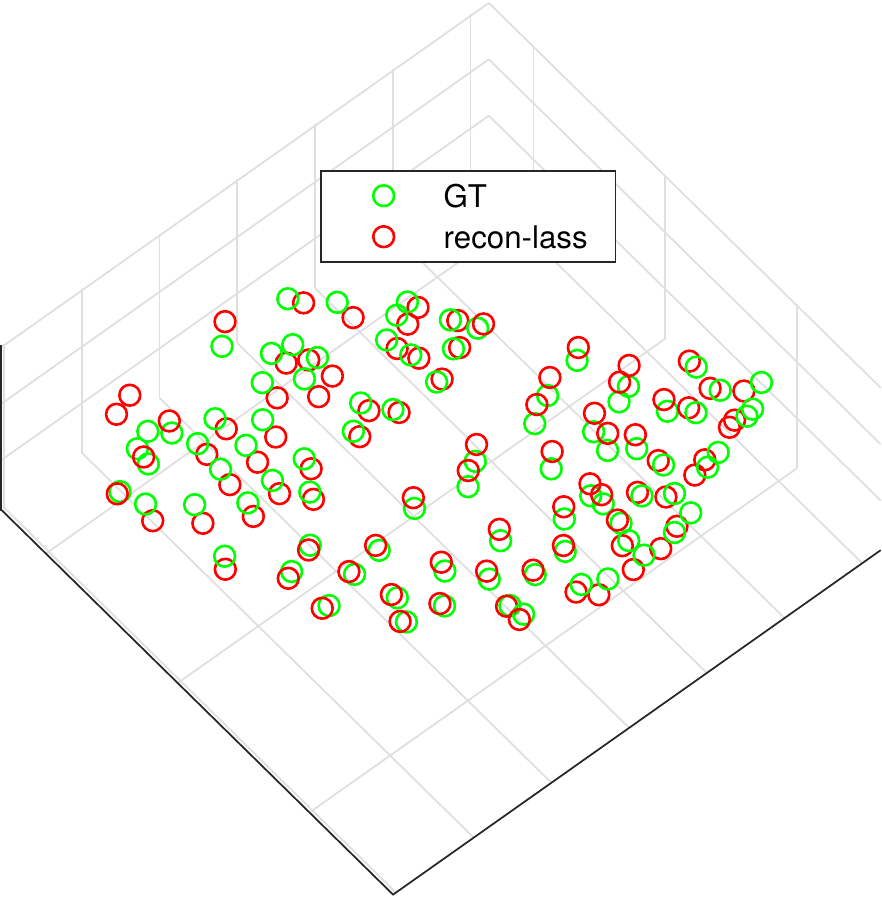}
    \includegraphics[width=5.5cm]{results/v03.eps}
    \includegraphics[width=5.5cm]{results/v04.eps}\\
    \includegraphics[width=5.5cm]{results/v05.eps}
    \includegraphics[width=5.5cm]{results/v06.eps}
    \includegraphics[width=5.5cm]{results/v07.eps}\\
    \includegraphics[width=5.5cm]{results/v08.eps}
    \includegraphics[width=5.5cm]{results/v09.eps}
    \includegraphics[width=5.5cm]{results/v10.eps}
    \caption{Nine views  reconstructions of tshirt dataset using Lassare's relaxation.}
    \label{fig:my_label}
\end{figure*}
\newpage
\begin{figure*}
\centering
\begin{minipage}[t]{1\textwidth}
    \begin{minipage}[h]{0.48\textwidth}
    \includegraphics[width=1\textwidth]{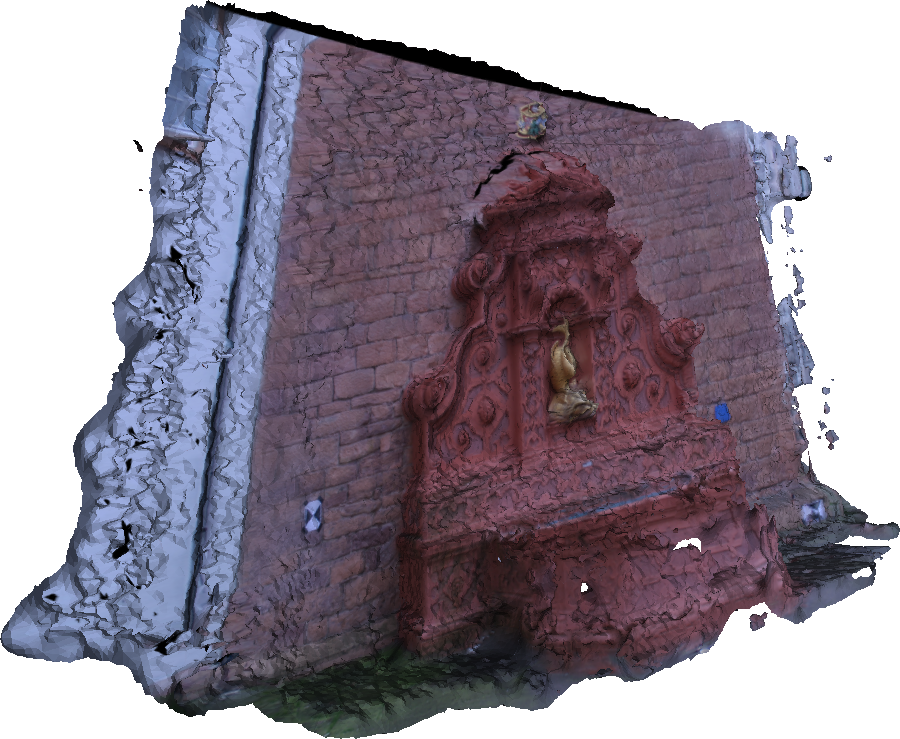}
    \end{minipage}
    \begin{minipage}[h]{0.48\textwidth}
    \includegraphics[width=1\textwidth]{results/fountain_screenshot01_c.png}
    \end{minipage}\\
    \begin{minipage}[h]{0.48\textwidth}
    \includegraphics[width=1\textwidth]{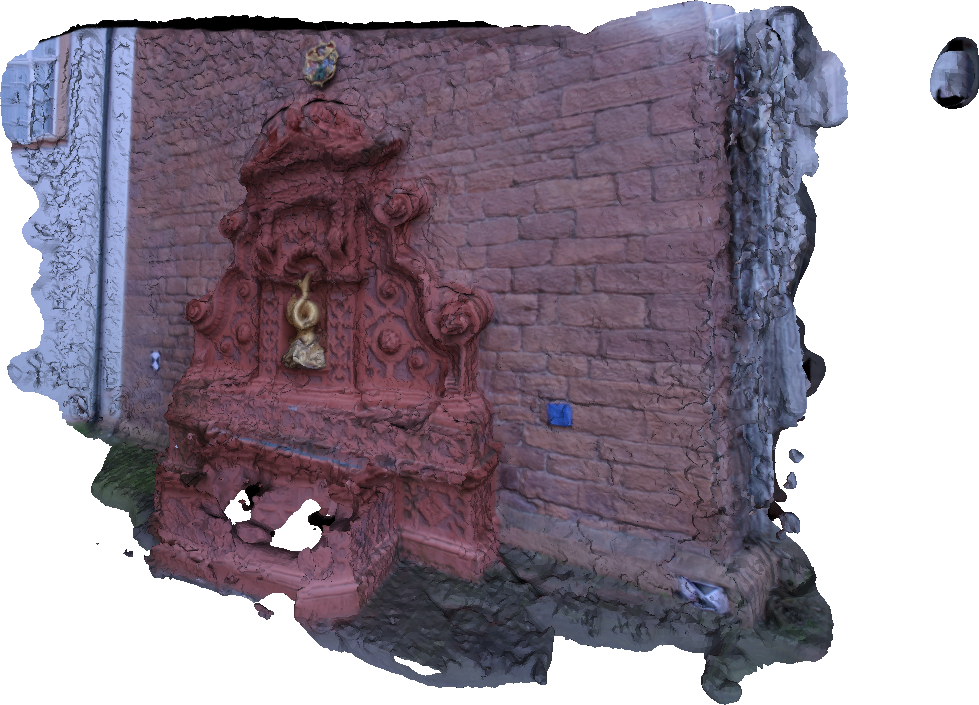}
    \end{minipage}
    \begin{minipage}[h]{0.48\textwidth}
    \includegraphics[width=1\textwidth]{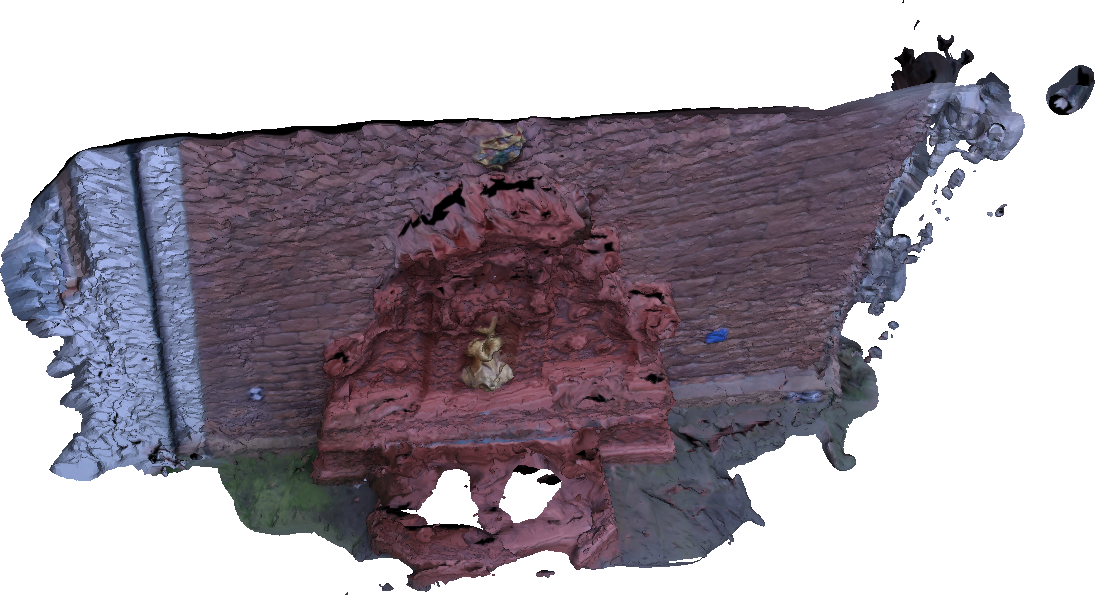}
    \end{minipage}
\end{minipage}
\caption{Rendered views for the 3D reconstruction of the Fountain sequence using the intrinsics obtained from our method.}
\label{fig:fountain}
\end{figure*}

\begin{figure*}
    \centering
    \includegraphics[width=8.5cm]{results/chicken1.eps}
    \includegraphics[width=8.5cm]{results/chicken2.eps}\\
        \includegraphics[width=8.5cm]{results/cloud1.eps}
    \includegraphics[width=8.5cm]{results/cloud2.eps}\\
            \includegraphics[width=8.5cm]{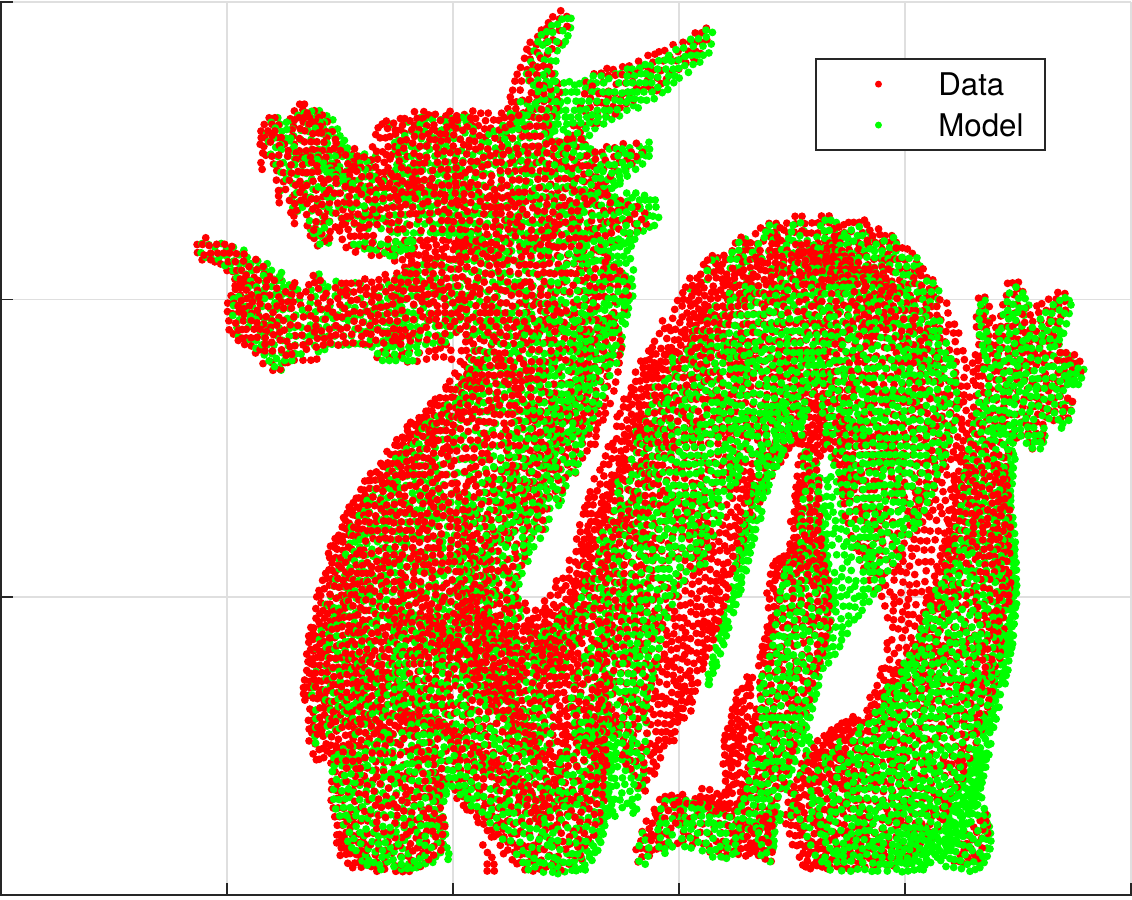}
    \includegraphics[width=8.5cm]{results/dragon2.eps}\\
    \caption{Qualitative results of the non-minimal registration (our method) on the publicly dataset~\cite{chin2016guaranteed}. Each row shows two views of the same registration obtainde by non-minimal solver. Illustration shows the model in green and the  data in red.}
    \label{fig:my_label}
\end{figure*}

\begin{figure*}
    \centering
            \includegraphics[width=8.5cm]{results/tree1.eps}
    \includegraphics[width=8.5cm]{results/tree2.eps}\\
            \includegraphics[width=8.5cm]{results/dino1.eps}
    \includegraphics[width=8.5cm]{results/dino2.eps}\\
                \includegraphics[width=8.5cm]{results/dino11.eps}
    \includegraphics[width=8.5cm]{results/dino12.eps}\\
  \caption{Qualitative results of the non-minimal registration (our method) on the publicly dataset~\cite{chin2016guaranteed}. Each row shows two views of the same registration obtainde by non-minimal solver. Illustration shows the model in green and the  data in red.}
    \label{fig:my_label}
\end{figure*}

\begin{figure*}
    \centering
            \includegraphics[width=8.5cm]{results/budda1.eps}
    \includegraphics[width=8.5cm]{results/budda2.eps}\\
            \includegraphics[width=8.5cm]{results/bush1.eps}
    \includegraphics[width=8.5cm]{results/bush2.eps}\\
                \includegraphics[width=8.5cm]{results/bush11.eps}
    \includegraphics[width=8.5cm]{results/bush12.eps}\\
  \caption{Qualitative results of the non-minimal registration (our method) on the publicly dataset~\cite{chin2016guaranteed}. Each row shows two views of the same registration obtainde by non-minimal solver. Illustration shows the model in green and the  data in red.}
    \label{fig:my_label}
\end{figure*}

%\begin{figure*}
%    \centering
%            \includegraphics[width=8.5cm]{results/hulk1.eps}
%    \includegraphics[width=8.5cm]{results/hulk2.eps}\\
%            \includegraphics[width=8.5cm]{results/branch1.eps}
%    \includegraphics[width=8.5cm]{results/branch2.eps}\\
%                \includegraphics[width=8.5cm]{results/branch11.eps}
%   \includegraphics[width=8.5cm]{results/branch12.eps}\\
%  \caption{Qualitative results of the non-minimal registration (our method) on the publicly dataset~\cite{chin2016guaranteed}. Each row shows two views of the same registration obtainde by non-minimal solver. Illustration shows the model in green and the  data in red.}
%    \label{fig:my_label}
%\end{figure*}